\documentclass[final,3p,times,twocolumn,authoryear]{elsarticle}
\usepackage[utf8]{inputenc} % make more characters printable when copy-pasted here
\usepackage[T1]{fontenc}

\usepackage{amssymb}
\usepackage{lineno}
\usepackage{amsmath,amsfonts,amsthm} % Math packages
\newtheorem{thm}{Theorem}[section]

\usepackage{graphicx}
\graphicspath{{./}{./figures/}}

\usepackage{xcolor}
\usepackage[colorlinks]{hyperref}
\AtBeginDocument{%
	\hypersetup{
		colorlinks=true,
		allcolors=purple,
}}

\usepackage{xifthen}% provides \isempty test
\newcommand{\comment}[2][]{{\color{red}
		\ifthenelse{\isempty{#1}}%
		{}% if #1 is empty
		{\textit{#1}:~}% if #1 is not empty
		#2}}

\usepackage{caption}
\usepackage{subcaption}

\usepackage[ruled,vlined]{algorithm2e}

\usepackage{multirow}
\usepackage{makecell}

\usepackage{float}
\usepackage{stfloats}

\usepackage{pifont}% http://ctan.org/pkg/pifont
\newcommand{\xmark}{\ding{55}}%

\journal{arXiv}

\begin{document}
	
	\begin{frontmatter}
		
		%% Title, authors and addresses
		\title{Fast algorithm for centralized multi-agent maze exploration}
		
		\author[mathri]{Bojan Crnković}
		\ead{bojan.crnkovic@uniri.hr}
		\author[riteh]{Stefan Ivić}
		\ead{stefan.ivic@riteh.hr}
		\author[unimostar]{Mila Zovko\corref{cor1}}
		\ead{mila.zovko@fpmoz.sum.ba}
		
		\cortext[cor1]{Corresponding author}
		
		\affiliation[mathri]{organization={Faculty of Mathematics, University of Rijeka},
			addressline={Radmile Matejčić 2}, 
			city={Rijeka},
			postcode={51000}, 
			%state={},
			country={Croatia}}
		\affiliation[riteh]{organization={Faculty of Engineering, University of Rijeka},
			addressline={Vukovarska 58}, 
			city={Rijeka},
			postcode={51000}, 
			%state={},
			country={Croatia}}
		\affiliation[unimostar]{organization={Faculty of Science and Education, University of Mostar},
			addressline={Matice hrvatske bb}, 
			city={Mostar},
			postcode={88000}, 
			%state={},
			country={Bosnia and Herzegovina}}
		
		\begin{abstract}
Recent advances in robotics have paved the way for robots to replace humans in perilous situations, such as searching for victims in burning buildings, in earthquake-damaged structures, in uncharted caves, traversing minefields or patrolling crime-ridden streets. These challenges can be generalized as problems where agents have to explore unknown mazes. We propose a cooperative multi-agent system of automated mobile agents for exploring unknown mazes and localizing stationary targets. The Heat Equation-Driven Area Coverage (HEDAC) algorithm for maze exploration employs a potential field to guide the exploration of the maze and integrates cooperative behaviors of the agents such as collision avoidance, coverage coordination, and path planning. In contrast to previous applications for continuous static domains, we adapt the HEDAC method for mazes on expanding rectilinear grids. The proposed algorithm guarantees the exploration of the entire maze and can ensure the avoidance of collisions and deadlocks. Moreover, this is the first application of the HEDAC algorithm to domains that expand over time. To cope with the dynamically changing domain,   succesive over-relaxation (SOR) iterative linear solver has been adapted and implemented, which significantly reduced the computational complexity of the presented algorithm when compared to standard direct and iterative linear solvers. The results highlight significant improvements and show the applicability of the algorithm in different mazes. They confirm its robustness, adaptability, scalability and simplicity, which enables centralized parallel computation to control multiple agents/robots in the maze.
		\end{abstract}

		%Graphical abstract
		
		\begin{keyword}
			%% keywords here, in the form: keyword \sep keyword
			cooperative search \sep multi-agent system \sep maze solving \sep maze exploration \sep graph exploration \sep succesive over-relaxation method
			%% PACS codes here, in the form: \PACS code \sep code
			
			%% MSC codes here, in the form: \MSC code \sep code
			%% or \MSC[2008] code \sep code (2000 is the default)
			
		\end{keyword}
		%Graphical abstract
%		\begin{graphicalabstract}
%			\includegraphics[width=\textwidth]{graficalabs.pdf}
%		\end{graphicalabstract}
		
		%Research highlights
%		\begin{highlights}
%			\item Autonomous multi-agent exploration of an unknown maze.
%			\item Based on the novel discrete version of Heat Equation Driven Area Coverage (HEDAC) algorithm.
%			\item Fast and lightweight implementation of an iterative linear solver which enables parallel computation.
%			\item Successfully tested and compared to similar algorithms.
%		\end{highlights}
	\end{frontmatter}
	
	%%
	%\linenumbers
	
	%% main text
	\section{Introduction}
	\label{sec:introduction}

	Maze exploration has captivated human curiosity over a long period of time. It has served as a valuable tool in scientific investigations, often to assess the cognitive abilities of animals, especially mice, and more recently to study the artificial intelligence of robots. Numerous papers have been published dealing with the exploration of a maze by a single agent or a multi-agent system. To illustrate the long-standing interest of scientists in this topic, it suffices to mention that the Trémaux algorithm dates back to the second half of the 19th century and one of the another first papers addressing maze exploration problem was \cite{shannon1993presentation}, which investigated the exploration of a two-dimensional, maze by one agent (a mechanical mouse) where the agent initially has no knowledge about maze structure.
	
	In this paper, we consider the problem of exploring an unknown maze by a cooperative multi-agent system that traverses the maze along continuous paths in discrete time. Each maze is defined by its nodes, which represent possible locations within the maze, and by walls that may separate neighboring nodes.
	To simulate a real-world application as in \cite{atilla2013design}, collisions, collaboration, coordination of coverage and path planning must be taken into account when controlling the agents.
	The common goal of maze exploration is to minimize the total time required to find the target or until all nodes are visited (i.e. maze mapping).
	The maze exploration problem can be divided into two different subtypes: those where the layout of the maze, including the walls and corridors, is already known, and those where it is unknown. Within these subtypes, there are additional variations depending on whether the target location (usually the exit) that agents are looking for in a maze is known, and whether or not these targets are stationary over time.
	If the layout of the maze is already known, agents usually rely on path planning algorithms to navigate through the maze and eventually reach their target, e.g. an exit.
	\cite{alamri2021autonomous} and \cite{sadik2010comprehensive} have mentioned some of the known algorithms for single agent maze exploration where the layout of the maze is known in advance, along with the start node and the target node, and of these we will briefly describe some. The dead-end fill algorithm works by tracing the walls of the maze and marking all dead ends or paths that do not lead to the exit. The flood-fill algorithm can also be modified to explore known mazes by assigning a value to each maze node based on its distance to the exit, as described by \cite{sadik2010comprehensive}.
	\cite{lee1961} proposed a path planning algorithm based on the "breadth first search" strategy, which is able to identify all possible paths between two maze nodes and select the shortest one. \cite{soukp1978} proposed an algorithm based on a combination of breadth first and depth first algorithms, which proved to be 10-50 times faster than Lee's algorithm. \cite{hadlock1977} introduced another faster variant of Lee's algorithm, which takes into account the value known as the detour number instead of the distance to the target as a cost measure.
	Several heuristic search algorithms are also used for this purpose, and one of the best known algorithms of this type is the A* algorithm  introduced by \cite{hart1968}. This algorithm combines depth-first search with a heuristic estimate of the cost of reaching the target node.
	The A* algorithm was significantly modified and adapted to various requirements. For example, \cite{warren1993} introduced a method for planning the path of an object (robot) through a space littered with obstacles by using a modified A* method to search through free space. \cite{ElHalawany2013} proposed a modified A* algorithm that considers the size of the robot as a parameter to generate a safer path for the robot and avoid sharp turns. In addition, \cite{le2018} presented a modified A* algorithm for efficient path planning for a Tetris-inspired self-configurable robot with an integrated laser sensor. One of the latest A* modifications is the fuzzy A* method presented in \cite{Airlangga_2023}, which uses the principles of fuzzy logic to adapt to different degrees of complexity of the environment.
	All the previously mentioned algorithms are primarily intended for single-agent path planning, and according to \cite{Tjiharjadi2022_1}, not many studies have used multi-agent path finding (MAPF) in mazes.
	One of the other approaches to the problem of path planning and maze exploration is based on the artificial potential field (APF) method, which was originally introduced in \cite{khatib1986real} with the aim of finding a collision-free path for the robot arm. The basic idea of generating an artificial potential field is to assign an attractive field to the target and a repulsive field to the obstacles in the robot environment. The combination of all repulsive and attractive fields results in the artificial potential field. The robot was considered as a particle that must move in this artificial potential field. The improvement of the APF method was continued (\cite{krogh1984improvement}), as well as its application (\cite{thorpe1984application}), which eventually led to the combined method for global and local path planning presented in \cite{krogh1986generalized}.
	Over time, these methods have also undergone some modifications that ensure near-optimal smooth paths and allow application to dynamic multi-agent target search. In \cite{vadakkepat2000evolutionary}, a new method called Evolutionary Artificial Potential Field (EAPF) was presented for real-time robot path planning. They combine the APF method with genetic algorithms to derive optimal potential field functions and ensure that the local minimum problem associated with EAPF is avoided. Simulation results have shown that the proposed methodology is robust and efficient for robot path planning when targets and obstacles are non-stationary. In \cite{yagnik2010motion}, a hybrid control methodology using APF and a modified simulated annealing optimization algorithm for motion planning of a team of multi-link snake robots is presented. A simulated annealing optimization algorithm is used for the robots to recover from local minima, while APF is used for simple and efficient path planning. This hybrid control method has proven successful in navigating the robot to its destination while avoiding collisions with other robots and obstacles.
	
	However, when exploring a known maze, the agent is already familiar with the structure of the maze and enters the maze with a predetermined route to the target. In this way, collisions, dead ends and similar problems that occur when exploring an unknown maze can be avoided. Real-world problems often cannot be reduced to the exploration of a known maze, as we cannot guarantee that the structure we perceive as a maze has not been altered beforehand. For these reasons, we tend to focus on algorithms for exploring unknown mazes.
	If the structure of the maze is unknown in advance, agents must first explore the immediate environment using sensors and then strategize their actions based on the information gathered. Navigating through an unknown maze in search of stationary goals, such as an exit, therefore involves problems related to detection, search and wayfinding and can be addressed using appropriate methods and techniques.
	\cite{alamri2021autonomous} and \cite{rao1993} list several basic algorithms for unknown maze exploration by a single agent, where the agent has no prior knowledge about maze structure but knows the starting position and the target location: Wall-Folower, Random Mouse, Tremaux algorithm, Tarry algorithm and Pledge algorithm.
	The Wall-Folower algorithm is the most commonly used algorithm for exploring mazes. The basic idea is that the agent follows the walls of the maze and applies the right or left rule whenever there is an intersection. This algorithm works well only on perfect mazes. 
	The random mouse algorithm is one of the simplest algorithms for solving mazes.  Due to its random movements, the algorithm can take a long time to find the target, especially in large or complex mazes. Tremaux's algorithm for exploring unknown mazes is one of the first systematic approaches to solving this problem. It is a simple approach to solving mazes based on path markers: The agent traverses the maze and marks each path it traverses; when it encounters an intersection, it chooses an unvisited path or one that has been marked fewer times. This process is repeated until the target in the maze is found. Similar to Tremaux's algorithm is Tarry's algorithm, which is based on the idea of constructing a cyclic directed path that traverses each node of the maze once and only once in each direction, and is described in more detail in \cite{rao1993}. Both Tremaux's algorithm and Tarry's algorithm are forms of depth-first search (DFS). The basic idea of the pledge algorithm is that the agent keeps the obstacle on its left or right side while exploring the maze and navigates by summing up the turning angles it encounters along the boundary of the obstacles \cite{abelson1981exploring}.
	\cite{elshamarka2012} presented a method based on the flood-fill algorithm for exploring an unknown maze by a single agent, where the initial position of the agent and the location of the target are known, but the agent initially has no information about the obstacles between them.
	\cite{oshada2023} proposed a new method for reducing the travel time through a maze whose wall structure is initially unknown to the micro-mouse robot, and it is based on the Flood-fill algorithm, Bézier curve interpolation, and path tracing techniques.
	One of the newer path-finding methods that could be used for maze navigation was presented in \cite{NAGHIZADEH2020113217}. This method is primarily intended for pathfinding of robots in unknown environments with unknown target positions and is based on bio-logical firing patterns of
	cells. This method proved to be fast and memory efficient for uninformed search, e.g. finding the target in a maze.
	Reinforcement learning (RL) is the process of learning based on a state-reward-action function. One of the best known RL algorithm for path planning is the Q algorithm, which is analyzed in detail in \cite{Jang2019}. This approach evolves significantly when integrated with neural networks, leading to the development of the Deep Q-Network (DQN) algorithm, which was used by \cite{nandy2023} for the purpose of solving the unknown maze by a single agent. The DQN algorithm uniquely assigns weights based on the results of previous trials, enabling a form of learning that is cumulative and adaptive. In contrast, algorithms such as A* treat all trials performed by the agent with equal importance and focus on predetermined heuristics to guide the search process. This fundamental difference emphasizes the adaptability of the DQN by learning from its environment and making adjusted decisions based on previous experience rather than relying solely on static rules or heuristics.
	Single-agent maze exploration algorithms are often used to solve mazes, but their extension to multi-agent systems is not optimal and leaves room for improvement.
	
	Some of the techniques mentioned were previously used for exploration by a single agent, but their modification allowed them to be used for exploration of the unknown maze by multiple agents. One of these modifications, specifically for the Tarry algorithm, is presented in \cite{kivelevitch2010multi}. The Kivelevitch-Cohen algorithm solves the problem of finding the exit in unknown perfect mazes, where a perfect maze.
	It was further modified and improved in \cite{alian2022multi} for general types of mazes and the implementation of collision avoidance of agents. These two algorithms are explained in more detail in Subsection \ref{sec:comparison_alternative}.
	In \cite{Tjiharjadi2022_2} it is demonstrated how the improved Flood-fill algorithm can be applied to find a target with two agents in a small unknown maze. The effectiveness of the algorithm when using two agents strongly depends on the initial position and orientation of the agents, since the agents can follow the same path, which can reduce the efficiency of the search.
	Another approach can be found in \cite{YOUSSEFI2021114907}, where a decentralized and asynchronous robot search algorithm based on particle swarm optimization is presented, which is used to solve mazes and find targets in unknown environments. This algorithm proved to be very effective in solving mazes of varying complexity.
	In \cite{husain2022}, it is presented how a bio-inspired technique, specifically ant algorithms (AA), can optimize the search time for a trapped victim in an unknown maze, and then Dijkstra's algorithm is used for the rescue phase. Four different ant algorithms were implemented to analyze and demonstrate the different effects of pheromones.
	\begin{table*}[h]
		\caption{The summary of features of maze exploration algorithms}
		\centering\footnotesize
		\begin{tabular}{cccccccc}
			\hline
			Reference & Year & Algorithm &Based & Known & Known & Additional & Multi \\
			&  & name &on the &  maze& target& requirements & agent \\
			&  & & algorithm& structure& location& &  \\
			\hline
			\cite{alamri2021autonomous}& &Random mouse& -& \xmark&\checkmark&-&\xmark \\
			\hline
			\cite{alamri2021autonomous}& 1881& Tremaux's  &DFS& \xmark&\checkmark&-&\xmark \\
			\hline
			\cite{alamri2021autonomous}& 1895 & Tarry's  &DFS& \xmark&\checkmark&-&\xmark \\
			\hline
			\cite{shannon1993presentation}& 1951 &Maze-solving machine &- &\xmark&\xmark &- &\xmark \\
			\hline   
			\cite{abelson1981exploring} &20th c.& Pledge &- & \xmark&\checkmark&-&\xmark \\
			\hline  
			\cite{alamri2021autonomous}& 20th c.&Dead-end filling  & - &\checkmark&\checkmark &- &\xmark  \\
			\hline
			\cite{alamri2021autonomous}& 20th c.&  Flood-fill  &- &\checkmark&\checkmark & -&\xmark  \\
			\hline
			\cite{lee1961}& 1961& Lee's& BFS&\checkmark&\checkmark &- &\xmark  \\
			\hline 
			\cite{hart1968}&1968 &A*+ DFS &heuristic&  \checkmark&\checkmark &- &\xmark  \\
			\hline
			\cite{hadlock1977}& 1977 &Hadlock's & Lee's&\checkmark&\checkmark &- &\xmark  \\ 
			\hline
			\cite{soukp1978}& 1978& Soukup's &BFS+ DFS& \checkmark&\checkmark & -&\xmark  \\
			\hline
			\cite{alamri2021autonomous}&20th c. &Wall follower & - & \xmark&\checkmark &perfect mazes &\xmark  \\
			\hline
			\cite{kivelevitch2010multi}& 2010 &Kivelevitch -Cohen& Tarry's  &\xmark&\xmark &perfect mazes &\checkmark \\
			\hline
			\cite{elshamarka2012}& 2012& Maze-Solving Robot &FF& \xmark&  \checkmark&-&\xmark \\
			\hline
			\cite{NAGHIZADEH2020113217}& 2020 & GridCell navigation & ANN  &\xmark&\xmark& -&\xmark \\
			\hline
			\cite{YOUSSEFI2021114907}& 2021 &Swarm intelligence &PSO  &\xmark& \xmark *& *beacon guided&\checkmark \\
			\hline
			\cite{alian2022multi}& 2022 & Alian's & Kivelevitch -Cohen &\xmark&\xmark &- &\checkmark \\
			\hline   
			\cite{Tjiharjadi2022_2}& 2022 & Improved Flood Fill& FF &\xmark& \checkmark &small maze&\checkmark \\
			\hline
			\cite{husain2022}& 2022 & Search and Rescue & AA+Dijkstra &\xmark& \xmark * &*beacon guided &\checkmark \\
			\hline
			\cite{oshada2023}& 2023&  Micromouse path planning&FF& \xmark&\checkmark &-&\xmark \\
			\hline
			\cite{nandy2023}& 2023 &Maze solving  & DQN &\xmark&\checkmark & -&\xmark \\
			\hline
			& 2024 & HEDAC maze exploration & HEDAC &\xmark&\xmark & -&\checkmark \\
			\hline
		\end{tabular}
		\label{tab:literature_overview_maze}
	\end{table*}
	It is therefore obvious that multi-agent exploration of unknown mazes represents a very interesting area that is approached with different methods. From Table ~\ref{tab:literature_overview_maze}, it is evident that there is not much research yet that focuses on exploring an unknown maze with unknown location of targets with a multi-agent system.
	For this reason, we present an algorithm for exploring an unknown maze by multiple agents with the goal of finding a hidden stationary target, e.g., an exit. Each agent explores a part of the maze and immediately communicates information about its current position, recognized walls and visited nodes.
	The proposed algorithm, hereafter referred to as the HEDAC maze exploration algorithm, is inspired by the HEDAC method presented in \cite{ivic2016ergodicity}. HEDAC is an ergodic multi-agent motion control algorithm that can achieve the given target density of generated agent trajectories. This algorithm has been modified and used in many different applications. The autonomous control for non-uniform spraying of multiple UAVs (unmanned aerial vehicles) presented in \cite{ivic2019autonomous} has proven that HEDAC-controlled UAV spraying swarms should significantly outperform UAVs operating with other existing path planning methods.
	A similar task is presented in \cite{low2022drozbot}, where a robotic system draws artistic portraits. The control of the drawing is represented as an ergodic coverage problem, which is then solved by a control algorithm based on HEDAC.
	UAV motion control techniques for search missions \cite{ivic2020motion} and \cite{ivic2022constrained} successfully utilize the HEDAC idea. Static and dynamic obstacle collision avoidance, non-radial sensing capabilities, irregular domains and heterogeneous UAV swarms are the most interesting improvements and extensions of HEDAC in these works.
	The application of HEDAC to trajectory planning for autonomous three-dimensional visual inspection of infrastructure with multiple UAVs is presented in \cite{ivic2023multi}. The inspection planning method proved to be flexible for different setup parameters and applicable to real inspection tasks. The distributed coverage control of multi-agent systems in uncertain environments using heat transfer equations presented in \cite{zheng2022distributed} is also based on HEDAC ideas. And one of the latest HEDAC applications is a full-body robot control method for exploring and investigating a specific region of interest, presented in \cite{bilaloglu2023whole}. This was achieved by extending HEDAC to applications where robots have multiple sensors all over the body (e.g. tactile skin) and utilize all sensors for optimal exploration of the given region. In the last paper,  \cite{bilaloglu2024tactile} an improved ergodic control strategy was presented that can be used with point clouds and is based on HEDAC. This technique facilitates closed-loop exploration by measuring the actual coverage using image processing. Unlike existing methods, it uses spectral acceleration to approximate the potential field from transient diffusion, which simplifies the process without the need for complex pre-processing and enables real-time closed-loop control. It has also been demonstrated that this new method can be successfully used to wash kitchenware with curved surfaces by cleaning the dirt detected by vision online.
	The algorithm presented in this paper is used to control the agents based on a potential field generated by modeling thermal conduction phenomena in a discrete maze. The algorithm provides a built-in cooperative agent behavior that includes collision avoidance, coverage coordination, and near-optimal path planning. In contrast to previous applications of the HEDAC method that investigate a non-variable continuous domain, this work considers a discrete dynamic system in both time and space with a variable spatial domain.
	In the proposed HEDAC application, the iterative Black-Red Successive Over-Relaxation (BR SOR) method has been used to compute the potential field, which brings a significant performance improvement compared to the techniques used in HEDAC in previous publications (finite difference, finite element and recently published spectral acceleration methods \cite{bilaloglu2024tactile}).
	Since in this work the domain is unknown and dynamically expands as the exploration of the maze progresses, using direct solvers requires not only solving the linear system but also initializing the system at each time step, which is very time consuming, while iterative solvers provide an elegant and effective solution to this problem since the  iterations can be easily extended to newly discovered nodes without the need to initialize the linear system.
	We should also emphasize that algorithms whose goal is to find the target, e.g. an exit, in an unknown maze are not necessarily suitable for visiting all maze nodes, i.e. mapping the entire unknown maze, and vice versa. HEDAC maze exploration algorithm is primarily designed to find a target in an unknown maze, but can also be successfully used to explore the entire unknown maze.
	Thus, we present not only a new method that represents a novelty in the approach to finding a hidden stationary target in an unknown maze by multiple agents, but also a new application of the HEDAC method that enables successful real-time exploration of an initially unknown and dynamically expanding discrete domain.
	The basic structure of the paper is as follows.
	The basic idea of the HEDAC method in 2D continuous space and time domain is described in Section \ref{sec:HEDAC}. The extension of the original algorithm and problem to the exploration of mazes and the implementation of the HEDAC method in the discrete space and time domain is presented in Section \ref{sec:formulation_discpr}. 
	In Subsection \ref{sec:exp_env}, the conditions and datasets on which we conducted the numerical tests are described.
	In Subsection \ref{sec:results_discussion}, we present the results of all the performed numerical tests. Specifically, we demonstrate the maze exploration process governed by the proposed HEDAC maze exploration algorithm. Additionally, we provide a comparison of the classical direct Gaussian solver and the iterative BR SOR linear solver used in the HEDAC algorithm for maze exploration.  The impact of the collision avoidance mechanism on the test results is also examined.
	In the same subsection, we compare the HEDAC algorithm for maze exploration with  with algorithms that have similar goals, presented in \cite{kivelevitch2010multi} and \cite{alian2022multi}. Furthermore, in this subsection we analyze the scalability of the HEDAC maze exploration algorithm, followed by the final Section \ref{sec_conclusion}, in which the conclusion is presented.

	\section{The HEDAC exploration in two-dimensional continuous domain}
	\label{sec:HEDAC}
	The HEDAC method is a centralized control of the motion of agents in the bounded two-dimensional domain $\Omega \subset \mathbb{R}^2$ with a Lipschitz continuous boundary. In this application, we will allow an expanding space domain such that $\Omega(t_i)\subset\Omega(t_j)\subset\Omega$ for $i<j$. We should emphasize that the domain changes dynamically as the maze is explored/discovered. For now, we consider the trajectories of the mobile agents as known and denote them by $\mathbf{z}_p : [0,t ] \to \mathbb{R}^2$, for $p= 1,2,\ldots, N$ where $N$ is the number of mobile agents. The coverage function $c(\mathbf{x},t)$ is cumulative and counts how often the location $\mathbf{x}$ has been visited up to time $t$.
	To achieve maximum coverage, the movement of the agent is controlled using a simple first-order kinematic motion model:
	\begin{equation}
		\frac{ d \mathbf{z}_p(t)}{ d t}=v_a \cdot \frac{\nabla u (\mathbf{z}_p(t),t)}{\left|\left|\nabla u (\mathbf{z}_p(t),t)\right|\right|},\quad p=1,\ldots N,
		\label{eq:motion_equation}
	\end{equation}
	with initial conditions
	$$
	\mathbf{z}_p(0)=\mathbf{z}_{0,p},\quad \mathbf{z}_{0,p}\in\Omega,\quad p=1,\ldots N ,
	$$
	where $v_a$ is the magnitude of the agent velocity and $u:\mathbb{R}^{3}\to \mathbb{R}$ is an attractive scalar field that represents a potential or a temperature and directs the agents via its gradient. It is obtained as a solution of the stationary heat equation:
	\begin{equation}
		\Delta u (\mathbf{x},t) = \alpha \cdot u(\mathbf{x},t) -s(\mathbf{x},t) 
		\label{eq:heat_equation}
	\end{equation}
	with the Neumann boundary condition
	\begin{equation}
		\frac{\partial u}{\partial \mathbf{n}} = 0,\textrm{ on }\partial \Omega.
		\label{eq:neumann_bc}
	\end{equation}
	Here $\Delta$ is a Laplace operator, $\mathbf{n}$ is outward normal, $\Delta u$  can be
	physically interpreted as heat conduction, while $\alpha u$, where $\alpha > 0$ in \eqref{eq:neumann_bc} represents a convective heat flow and it governs cooling over the entire space domain. With increasing convective cooling, the temperature field $u$ tends to approach the source field, so that more details about the uncovered areas become known. Therefore, an increase in $\alpha$ leads to an improvement in local coverage. The source term $s$ is a non-negative spatial field and can be formulated in different ways, for example as in \cite{ivic2016ergodicity,ivic2019autonomous,ivic2020motion}, but for this application it is defined as
	$$
	s(\mathbf{x},t)=max(0,1-c(\mathbf{x},t)).
	$$
	It emphasizes the areas insufficiently covered up to time $t$. The main goal of this method is to control the movement of agents so that the source converges to 0 everywhere, which corresponds to the entire domain being explored, or until the target is found.
	\section{Formulation of a discrete maze exploration}
	\label{sec:formulation_discpr}
	We consider a simple maze with uniformly wide corridors and simple branches. The maze can be represented as a 2D domain, where the walls of the maze are modeled by Neumann boundary conditions. The maze is initially unknown and may have a hidden target (depending on the problem we want to solve). There are inherent difficulties with variable space domains that make this problem difficult and computationally expensive, such as repeatedly remeshing the numerical domain. We show that we are able to use a multi-agent system driven by a HEDAC-based centralized algorithm to fully map the maze or find the target. However, instead of the real 2D problem presented in Section \ref{sec:HEDAC}, we solve a simplified version of it, which is obtained by discretization of the space and time domain. This makes the algorithm simple and it is easy to track the dynamic changes in the $\Omega(t)$ domain.
	\subsection{Maze domain}
	
	\begin{figure}[htb!]
		\centering
		\begin{subfigure}[b]{0.4\textwidth}
			
			\includegraphics[width=\textwidth]{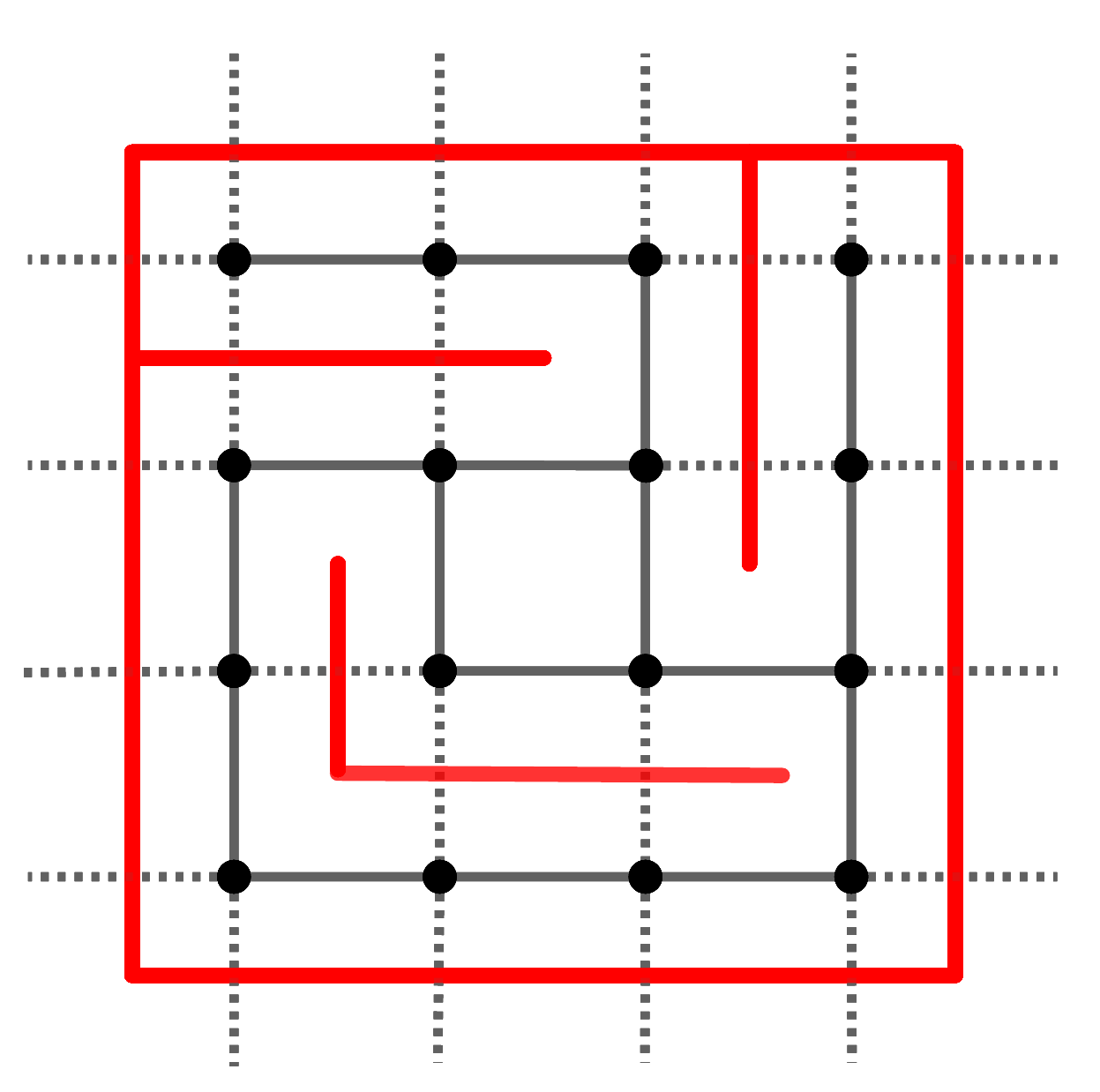}
			\caption{Maze structure}
			\label{fig:maze1}
		\end{subfigure}	
		\begin{subfigure}[b]{0.4\textwidth}
			
			\includegraphics[width=\textwidth]{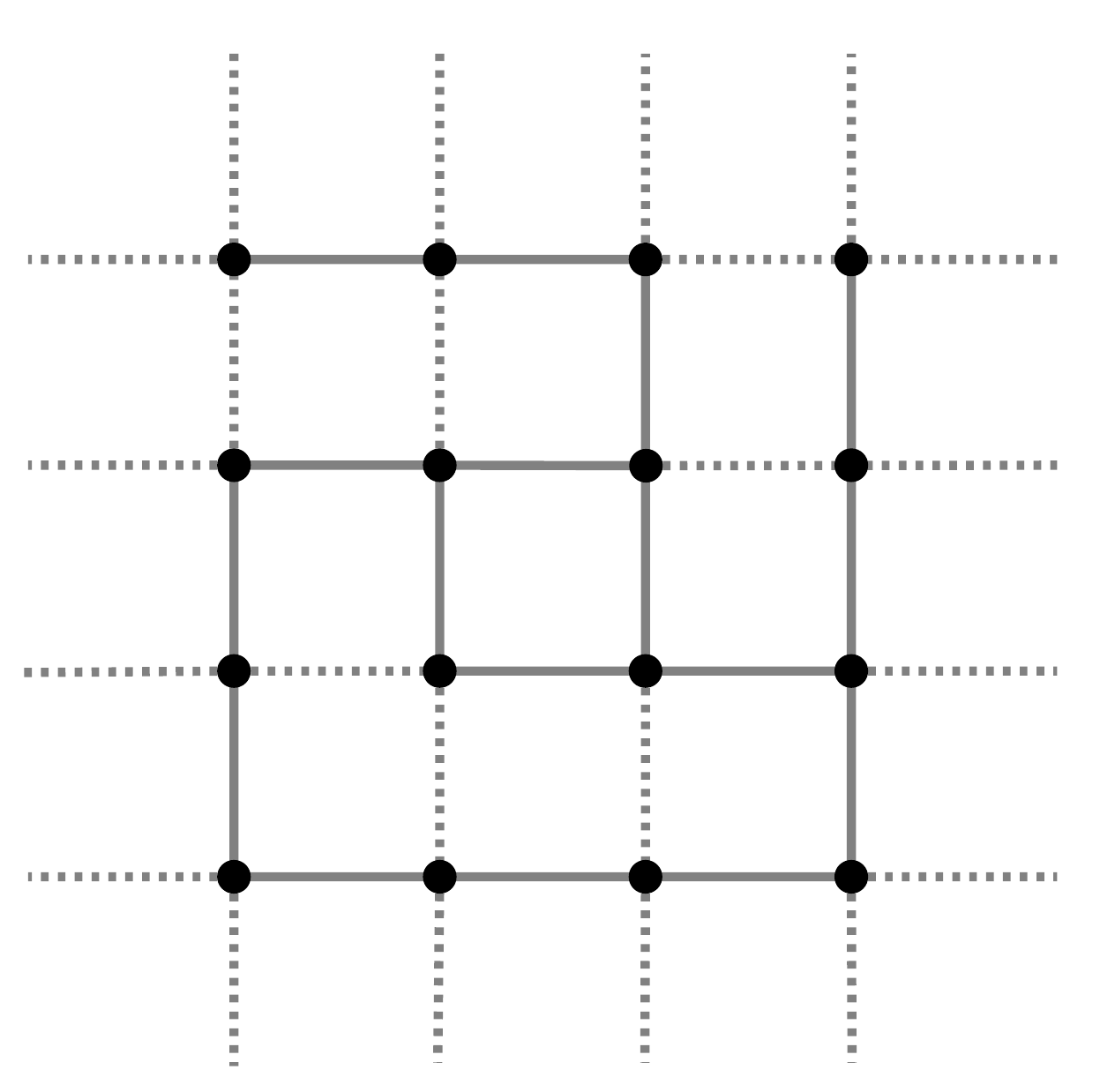}
			\caption{Corresponding numerical grid}
			\label{fig:maze2}
		\end{subfigure}
		
		\caption{An example of a maze, represented as a numerical grid. Maze walls are represented by a red line, and a possible location in the maze where an agent can stand is represented by black nodes. The edges between these nodes are only used to provide information about whether the two nodes are adjacent or not. A solid black line represents a real connection between two nodes, and a dotted black line represents an obstacle between two nodes.}
		\label{fig:maze}
	\end{figure}
	To simplify the application, we consider only a simple maze with orthogonal corridors in which we can place the nodes of the discrete numerical grid in a 2D coordinate system with integer coordinates. Each node of the numerical grid in Figure \ref{fig:maze} represents a position that an agent can occupy. The agent can move orthogonally around a grid node at each time step if the surrounding maze barriers allow this. The distances between neighboring nodes are all the same, so that the agent's velocity remains constant.
	\subsection{Agent motion control}
	\label{sec:control}
	Since we focus on the exploration of the maze and the cooperation of the agents, in this paper we use the simplest first-order kinematic model for the agents' motions, which neglects the mass and inertia of the agents and is constrained to orthogonal motion. We also introduce a new method for collision avoidance when moving through the maze. This algorithm allows agents to stop and wait until another agent is no longer blocking their motion. We have also left the option to allow collisions if an application of the algorithm requires it.
	At the beginning, the maze is unknown and the main priority of the agent system is to explore the maze. The agents explore the surroundings of their starting position with limited knowledge and focus exclusively on neighboring nodes that are known to them. The agents begin their movement from their starting position and their main goal is to search the unknown maze and find the target. They can only move one step up, down, left or right if there is no wall preventing them from doing so. This means that the directions in \eqref{eq:motion_equation} are limited to only 4 possibilities, and the agent chooses the one with the greatest potential (which corresponds to the gradient $\nabla\mathbf{u}$ from \eqref{eq:motion_equation}). The agents cannot go or see through walls, but they can exchange information about their position and the explored parts of the maze.
	The time domain is discrete $t_i = i, i=0,\ldots,n$. After each time step, the numerical domain and the set of indices $I(t_k)$ are updated with discovered  nodes. Therefore, the trajectories of the agents can be described by functions $z_p:\mathbb{N}_0 \to \mathbb{N}$, for $p=1,2,\ldots,N$.
	Since the space and time domains are discrete, we can represent the coverage function $c(x_{i,j},t_k)$ at time step $t_k$ for the maze node $x_{i,j}$ as a cumulative sum of ones and zeros depending on how many times the agents have visited the node $x_{i,j}$ up to time step $t_k$. For each agent that visits the maze node $x_{i,j}$ in each time step, 1 is added to the cumulative sum.
	In this discrete case, the source term $s$ is represented as a non-negative scalar field
	$$
	s(x_{i,j},t_k)=max(0,1-c(x_{i,j},t_k)) \cdot S(t_k)
	$$
	which emphasizes the insufficiently covered maze nodes $x_{i,j}$ at time step $t_k$, where $S(t_k)$ is a scaling factor that can be chosen to help the convergence of a linear solver. In our numerical examples, we have used $S(t_k)=1$.
	
	\begin{figure}[h!]
		\begin{subfigure}[b]{0.4\textwidth}
			\includegraphics[width=\textwidth]{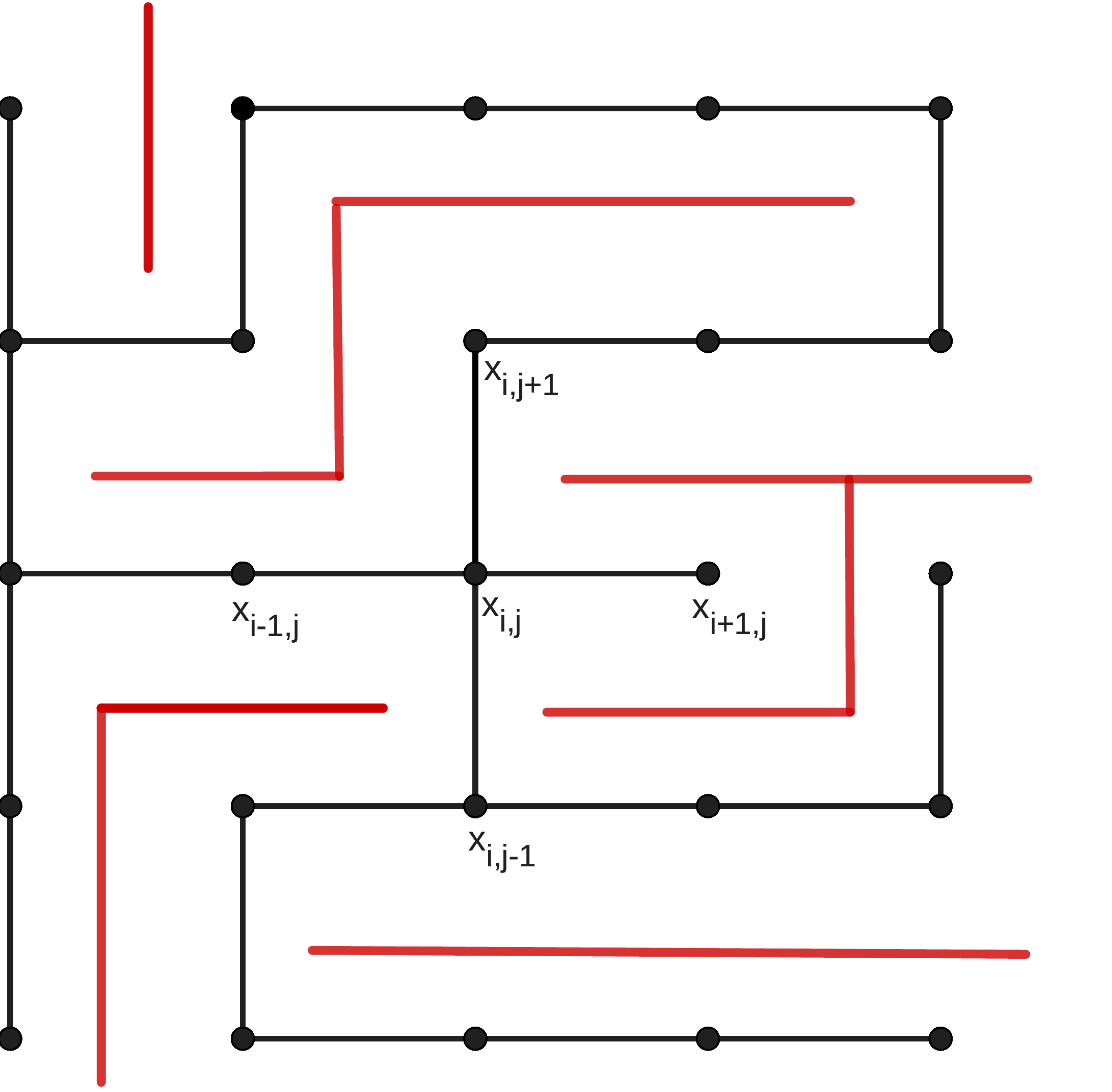}\centering
			\caption{Maze node  $x_{i,j}$ with all four  neighbors  }
			\label{fig:graph4}
		\end{subfigure}
		\begin{subfigure}[b]{0.4\textwidth}\centering
			\includegraphics[width=\textwidth]{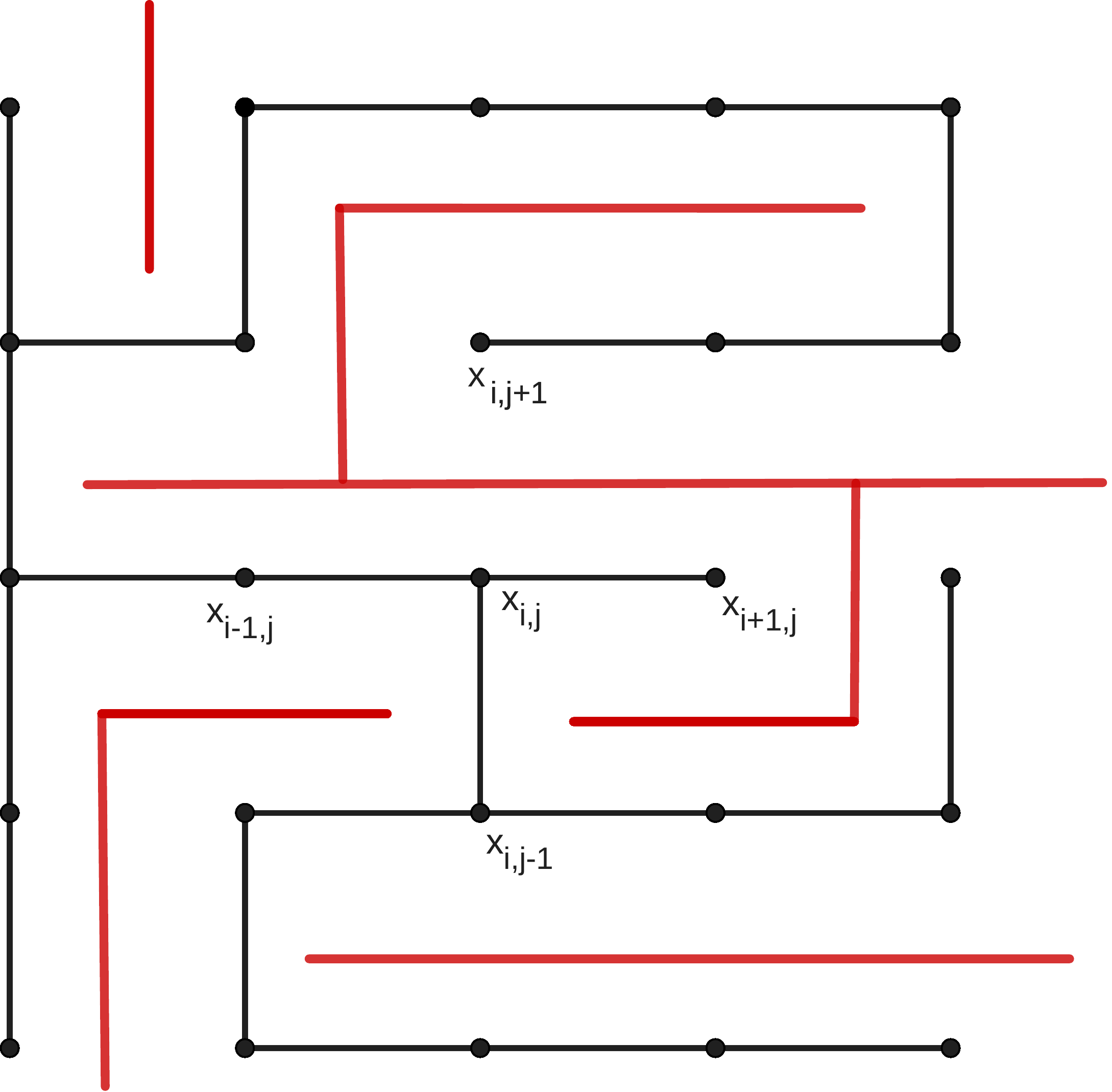}
			\caption{Maze node  $x_{i,j}$ with three  neighbors  }
			\label{fig:graph3}
		\end{subfigure}
		\caption {Examples of the maze nodes $x_{i,j}$ neighborhood depending on the arrangement of the maze walls. Maze  walls are shown with red lines, while the nodes and the connections between them are shown in black.} 
		\label{fig:xneighbours} 
	\end{figure}
	Figure~\ref{fig:graph4}, shows an example of a maze node $x_{i,j}$ that has no obstacles in the local neighborhood, so that the value $u_{i,j}$ of the solution of the PDE equation \eqref{eq:heat_equation} can be obtained by discrete finite difference approximation as
	\begin{multline*}\frac{u_{i-\delta x, j}-2u_{i, j}+u_{i+\delta x, j}}{(\delta x)^2} +\frac{u_{i, j-\delta y}-2u_{i, j}+u_{i, j+\delta y}}{(\delta y)^2}=\\
		=
		\alpha \cdot u_{i,j}-s(x_{i,j},t_k)\end{multline*}
	because we use $\delta x=\delta y=1$ we get
	\begin{equation}
		(4+\alpha)u_{i,j} - u_{i-1,j}-u_{i+1,j}-u_{i,j-1}-u_{i,j+1} =s(x_{i,j},t_k).
		\label{eq:iteration1}
	\end{equation}
	It remains to explain what to do with the more complicated case where an obstacle exists between a node and at least one of its neighboring nodes, as in Figure~\ref{fig:graph3}. Given the structure of the maze shown in Figure~\ref{fig:maze}, we can consider these disconnections/walls between the nodes as boundaries and impose Neumann boundary conditions there, since the Neumann boundary condition actually represents an ideal insulating wall from the point of view of the conductivity model, a barrier in the domain through which no heat is conducted.
	In the particular case in Figure~\ref{fig:graph3}, we can impose the Neumann boundary condition on the top of the node at position $(i,j)$. Second-order approximation of the Neumann boundary condition
	$$\frac{\tilde{u}_{i, j+\delta y}-u_{i, j-\delta y}}{2\delta y}=0$$
	where $\tilde{u}$ marks a temporary node, modifies the \eqref{eq:iteration1} to obtain
	
	\begin{equation}
		(4+\alpha)u_{i,j}-u_{i-1, j}-u_{i+1, j}-2u_{i, j-1}=s(x_{i,j},t_k).
		\label{eq:iteration2}
	\end{equation}
	There is another interesting case in Figure~\ref{fig:graph3} concerning the node at position $(i,j+1)$, we impose a Neumann boundary condition on the top, bottom and left side of the node. Due to the Neumann boundary condition, the partial derivative vanishes in the vertical direction and the equation reduces to a 1D case:
	\begin{equation}
		(2+\alpha)u_{i,j+1}-2u_{i+1, j+1}=s(x_{i,j},t_k).
		\label{eq:iteration3}
	\end{equation}
	
	All special cases \eqref{eq:iteration1},\eqref{eq:iteration2},\eqref{eq:iteration3} can be written in a general form that applies to all nodes:
	\begin{equation}
		(\sum_{l\in L} a_{i,j}(l)+\alpha)u_{i,j} - \sum_{l\in L} a_{i,j}(l)u_l =s(x_{i,j},t_k),
		\label{eq:iterationGeneral}
	\end{equation}
	where $L=\{(i+1,j),(i-1,j),(i,j+1),(i,j-1)\}$ is set of possible neighbor indices for maze node $x_{i,j}$. The interface function $a_{i,j}(l)\in \{0,1,2\}$ depends on the structure of the walls around the node $(i,j)$ and $0\leq \sum_{l\in L} a_{i,j}(l) \leq 4$ applies in general. In the case we are dealing with, where each maze node has at least one neighbor, holds $2 \leq \sum_{l\in L} a_{i,j}(l) \leq 4$ .The equation \eqref{eq:iterationGeneral} is valid for nodes with all four neighboring nodes and also in the absence of some neighboring nodes by applying Neumann boundary conditions to the finite difference approximation of the derivatives in \eqref{eq:heat_equation}. The solution of this linear system is a scalar field $u,$ which is used to control the agents in the maze.
	Each agent moves one step at time $t_k$ to a neighboring node where the scalar field $u$ has a larger value. A simple iterative numerical method can be used to approximate the solution by successively averaging over each node:
	\begin{equation}
		u^{r+1}_{i,j} = \frac{\sum_{l\in L} a_{i,j}(l)u^r_l+s(x_{i,j},t_k)}{\sum_{l\in L} a_{i,j}(l)+\alpha},
		\label{eq:iteration}
	\end{equation}
	where $r$ counts the iterations until convergence.
	Equation \eqref{eq:iteration} describes an iterative Jacobi method. The iterations are very simple and can be performed in parallel using the cumulative computing power of the multi-agent system. Simple iterations allow us to solve a growing domain without re-meshing, as new numerical nodes are simply added to a list and assigned to one of the processors with the lowest load.
	
	\begin{thm} \label{thm:jacobi}The iterations defined in \eqref{eq:iteration} converge to a unique solution $\tilde{u}$ for arbitrary initial values of $u$. In addition, the local maximums of $\tilde{u}$ are located in the unvisited node.
	\end{thm}
	\begin{proof}
		Equation \eqref{eq:iteration} is a Jacobi iterative solver for a linear system $Au=s$ which comes from a finite difference approximation of equation \eqref{eq:heat_equation}. The matrix of the linear system $A$ is strictly diagonally dominant for $\alpha>0$ from which it follows that the solution is unique and the Jacobi iteration converges for arbitrary initial vector $u$.
		
		Let's suppose that $u$ attains a local maximum value at some node ${i,j}$ which was visited before time $t_k$ it follows that $s(x_{i,j},t_k)=0$. From \eqref{eq:iteration} it follows	
		$$u_{i,j}=\frac{\sum_{l\in L} a_{i,j}(l)u_l}{\sum_{l\in L}a_{i,j}(l)+\alpha}<\frac{\sum_{l\in L} a_{i,j}(l)u_l}{\sum_{l\in L}a_{i,j}(l)}\leq\max_{l\in L} (u_l)$$ 
		which is a contradiction. 
		It follows that the local maximum is attained at some unvisited node.
	\end{proof}
	
	Theorem \ref{thm:jacobi} shows that simple iterations given by formula \eqref{eq:iteration} always converge and can be used to calculate the potential field $u$. The potential filed is used by the agents to navigate through the maze following standard steepest ascent hill climbing algorithm. Potential deadlock would happen only if the agent somehow reached a local maximum but this is impossible because   
	the local maximum can't be located in the visited node. The agents move and follow the steepest ascent algorithm and they will continue to do so until they have visited all available nodes, since the local maximum is constantly shifting towards the unvisited nodes. So, the Theorem \ref{thm:jacobi} guaranties the maze will be fully explored if all nodes can be reached by the agents.
	
	In addition, the iterations \ref{eq:iteration} are very simple, which simplifies the update of the linear system at each time step, so that the growth of the spatial domain is easy to implement and the memory requirements are very low.
	A disadvantage of this approach is that the convergence of the Jacobi method slows down for a large numerical domain. This convergence can be improved by Gauss-Seidel iterations, using freshly computed values of the state vector for each component of the state vector $u$. Although Gauss-Seidel should improve convergence, the parallelization of the iterations is not as straightforward as with the Jacobi method. In our implementation, we use the Successive Over-Relaxation (SOR) method, which improve the convergence of the Gauss-Seidel method but depend on the over-relaxation parameter $\omega$.  More specifically, we employ the Black-Red Successive Over-Relaxation (BR SOR) method to facilitate the parallelization of computations. This is achieved by partitioning the discovered nodes into two sets following a black-red strategy.
	We form two sets of nodes, black,  $B=\{ (i,j)\in I(t_k) | i+j\textrm{(mod 2)}=0\}$ and, red,  $R=\{ (i,j)\in I(t_k) | i+j\textrm{(mod 2)}=1\}$. So according to this definition of sets $B$ and $R$,  neighbors of any element from $B$ are not contained in set $B$ but in set $R$, and vice versa. And because of such a definition of sets $B$ and $R$, we can compute one iteration of the BR SOR method in two consecutive steps:
	\begin{multline}
		u^{r+1}_{i,j} =u^r_{i,j}+ \omega \left(\frac{\sum_{l\in L} a_{i,j}(l)u^r_l+s(x_{i,j},t_k)}{\sum_{l\in L} a_{i,j}(l)+\alpha} -u^r_{i,j}\right),\\ \textrm{for all } (i,j)\in B	
		\label{eq:BackiterationSOR}
	\end{multline}
	\begin{multline}
		u^{r+1}_{i,j} =u^r_{i,j}+ \omega \left(\frac{\sum_{l\in L} a_{i,j}(l)u^{r+1}_l+s(x_{i,j},t_k)}{\sum_{l\in L} a_{i,j}(l)+\alpha} -u^r_{i,j}\right),\\ \textrm{for all } (i,j)\in R			
		\label{eq:RediterationSOR}
	\end{multline}
	where $r$ represents the iteration number. The iterations in the steps \eqref{eq:BackiterationSOR} and \eqref{eq:RediterationSOR} are very simple, and each of them can be calculated in parallel. The optimal choice of the parameter $1<\omega < 2$ improves the convergence by two orders of magnitude compared to the Jacobi method.
	The optimal value of the relaxation parameter that maximizes the convergence of the SOR, as well as BR SOR method depends on the problem under consideration and the size of the domain, but is typically $1.5<\omega < 2.$
	To the best of our knowledge, determining the optimal relaxation parameters for the BR SOR method (given with \eqref{eq:BackiterationSOR} and \eqref{eq:RediterationSOR}) applied to the modified Helmholtz  equation \eqref{eq:heat_equation} with Neumann boundary conditions  on a rectangular grid remains an open problem. The optimal relaxation parameter for SOR method applied to the
	Poisson equation  with Dirichlet boundary conditions on a regular grid in two or more spatial dimensions can be found in \cite{watkins2004fundamentals} and \cite{yang2009optimal}. This can serve as a motivation for determining the theoretical optimal relaxation parameter for the SOR method applied to the modified Helmholtz equation.  We determined the over-relaxation parameter experimentally and found that its optimal value for the presented problem is $1.05<\omega < 1.75$ and depends on the  parameter $\alpha$ and the size of the spatial domain. If the relaxation parameter is chosen appropriately, the computational effort of the applied iterative method is competitive with commonly  used iterative solvers.
	Now that we know how $u$ is determined, we can write down the rules and the Algorithm \ref{alg:1} for controlling the movement of the agents:
	\begin{itemize}	 
		\item At the beginning, the maze is unknown, the agents appear at their initial positions, a random node is set as the exit of the maze, and the basic task of the agent system is to explore the maze and find the exit node.
		\item The agents have a defined order in which they choose their next move. An agent that has an associated trajectory $ \mathbf{z}_1$ at a given time $t$ decides first to which node it moves, while an agent that is associated with a trajectory $ \mathbf{z}_N$ decides last to which node it moves.
		\item The agent can move one step up, down, left or right.
		\item At time $t$ the agent moves to the neighboring node with the largest value of $u$.
		\item In addition, the agent has the information whether one of the other agents is currently standing on one of the neighboring nodes and does not consider this node for its next position if the anti-collision condition, hereinafter referred to as AC, is switched on. If the agent has no better option, it can wait in place until its path is clear.
		\item When the agent's new position is determined, the source at this node is set to $0$, so that agents that have not yet taken a step during time $t$ can determine their new positions based on the updated information (right side of the linear system \eqref{eq:BackiterationSOR},\eqref{eq:RediterationSOR}).
		\item At the end of time $t$, the agents update 
		their newly discovered nodes so that at the beginning of time $t+1$ complete data about visited and newly discovered nodes can be exchanged between all agents.
		\item The search ends when one of the agents has found the starting node or the maze has been fully explored, i.e. $\sum_{I(t)} S_{i,j} = 0$.
	\end{itemize}
	\begin{algorithm}[h!]
		\SetAlgoLined
		\footnotesize
		initialization \;	
		\For {$t=0, \, t= \max_t,\, t=t+1 $}{
			exchange of information about newly discovered nodes between agents; \\  
			update of a linear system \eqref{eq:BackiterationSOR},\eqref{eq:RediterationSOR};\\ 
			
			\For {$i=1, \, i=N, \, i=i+1$}{		
				iteratively solve a linear system \eqref{eq:BackiterationSOR},\eqref{eq:RediterationSOR};\\  
				next$\_$pos for agent $i$ = non ocuppied neighbouring nodes where is the largest value of the scalar field $u$ or current position if all neigbouring nodes are ocuppied;\\			
				\If{ new position of agent $i$ was previously unvisited}{update of right side of linear system \eqref{eq:BackiterationSOR},\eqref{eq:RediterationSOR} }
			}
			\eIf{ target  found }{break\;}
			
		}
		\caption{ \footnotesize Maze exploration with collision avoidance mechanism}
		\label{alg:1}
	\end{algorithm}

	\section{Performance analysis}
	
	\subsection{Experimental environment}
	\label{sec:exp_env}
	
	\begin{table*}[!hb]	
		\centering\footnotesize
		\caption{Summary of total number of simulations of maze exploration}
		\label{tab:experimental_environment}
		\begin{tabular}{lccccccccccc}
			\hline%\cline{2-5}
			\multicolumn{2}{l}{Maze shape} & \multicolumn{2}{c}{10 $\times$ 10} & \multicolumn{2}{c}{20 $\times$ 20} & \multicolumn{2}{c}{50 $\times$ 50} & \multicolumn{2}{c}{100 $\times$ 100} & \multicolumn{1}{c}{200 $\times$ 150} & \multicolumn{1}{c}{400 $\times$ 150} \\
			\multicolumn{2}{l}{Number of agents} & \multicolumn{2}{c}{1, 2, 3, 4, 5} & \multicolumn{2}{c}{1, 3, 5, 10} & \multicolumn{2}{c}{5, 10, 15} & \multicolumn{2}{c}{10, 20} & \multicolumn{1}{c}{20} & \multicolumn{1}{c}{20} \\
			%\hline
			\multicolumn{2}{l}{Wall density} & 45\%*  & 30\%  & 47.5\%* & 30\% & 40\% & 30\% & 40\% & 30\% & 30\% & 30\% \\  
			\multicolumn{2}{l}{Number of maze layouts} & 20  & 20  & 20 & 20 & 10 & 10 & 10 & 10 & 10 & 10\\  
			\hline
			Algorithm & \makecell{No. of\\initial\\config.} \\
			\hline
			HEDAC, AC on, DGESV & 5 & 500 & 500 & 400 & 400 & 150 & 150 & - & - & - & -\\
			HEDAC, AC off, DGESV & 5 & 500 & 500 & 400 & 400 & 150 & 150 & - & - & - & -\\
			HEDAC, AC on, SOR & 5 & 500 & 500 & 400 & 400 & 150 & 150 & 100 & 100 & 50 & 50\\
			HEDAC, AC off, SOR & 5 & 500 & 500 & 400 & 400 & 150 & 150 & 100 & 100 & 50 & 50\\
			HEDAC, AC on, SOR, known & 5 & 500 & 500 & 400 & 400 & 150 & 150 & - & - & - & -\\
			HEDAC, AC off, SOR, known & 5 & 500 & 500 & 400 & 400 & 150 & 150 & - & - & - & -\\
			Kivelevitch and Cohen, AC on  & 5 & 500 & 500 & 400 & 400 & - & - & - & - & - & -\\
			Kivelevitch and Cohen, AC off  & 5 & 500 & 500 & 400 & 400 & - & - & - & -& - & -\\
			Alian, AC on  & 5 & 500 & 500 & 400 & 400 & - & - & - & - & - & - \\
			\hline		
			\multicolumn{12}{l}{ * \scriptsize maximum possible wall density for given maze shape resulting with a perfect maze layouts}
		\end{tabular}
	\end{table*}
	\begin{figure*}[!hb]
		\centering
		\includegraphics[width=0.33\linewidth]{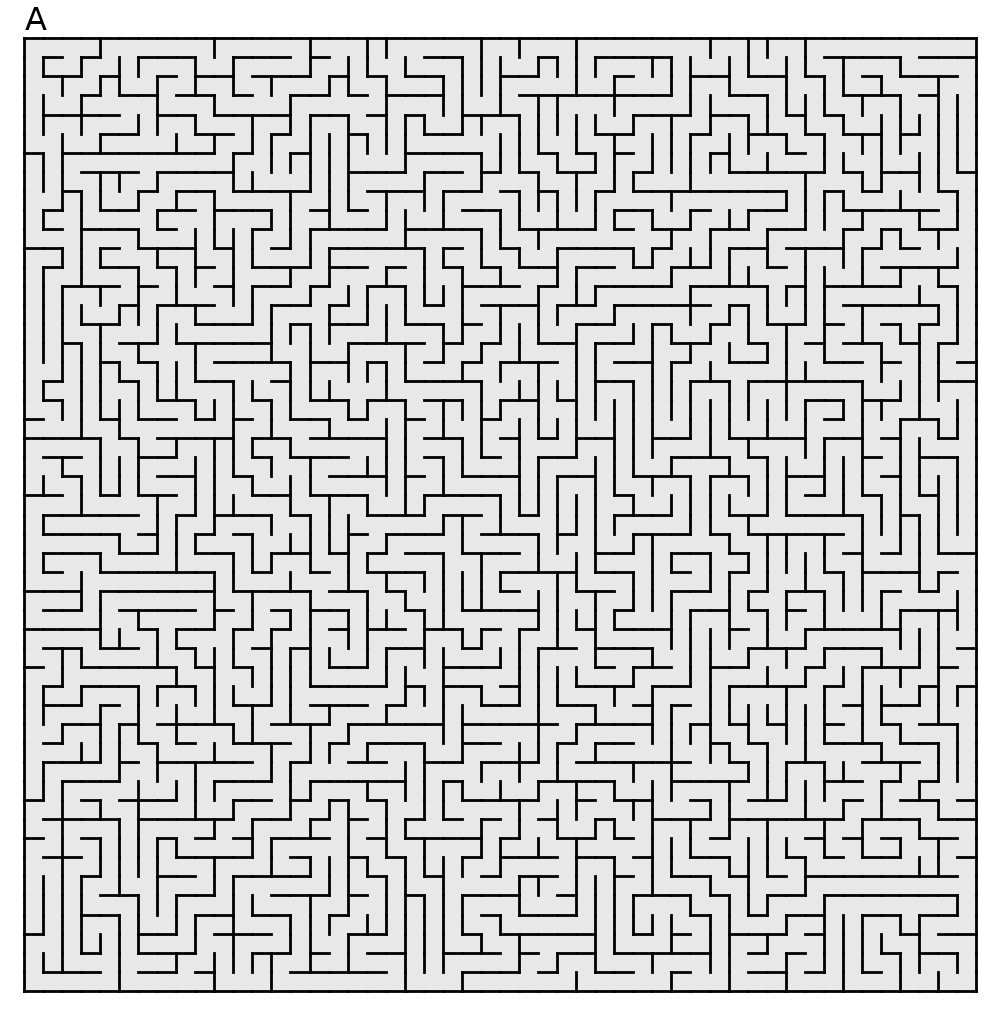}
		\includegraphics[width=0.33\linewidth]{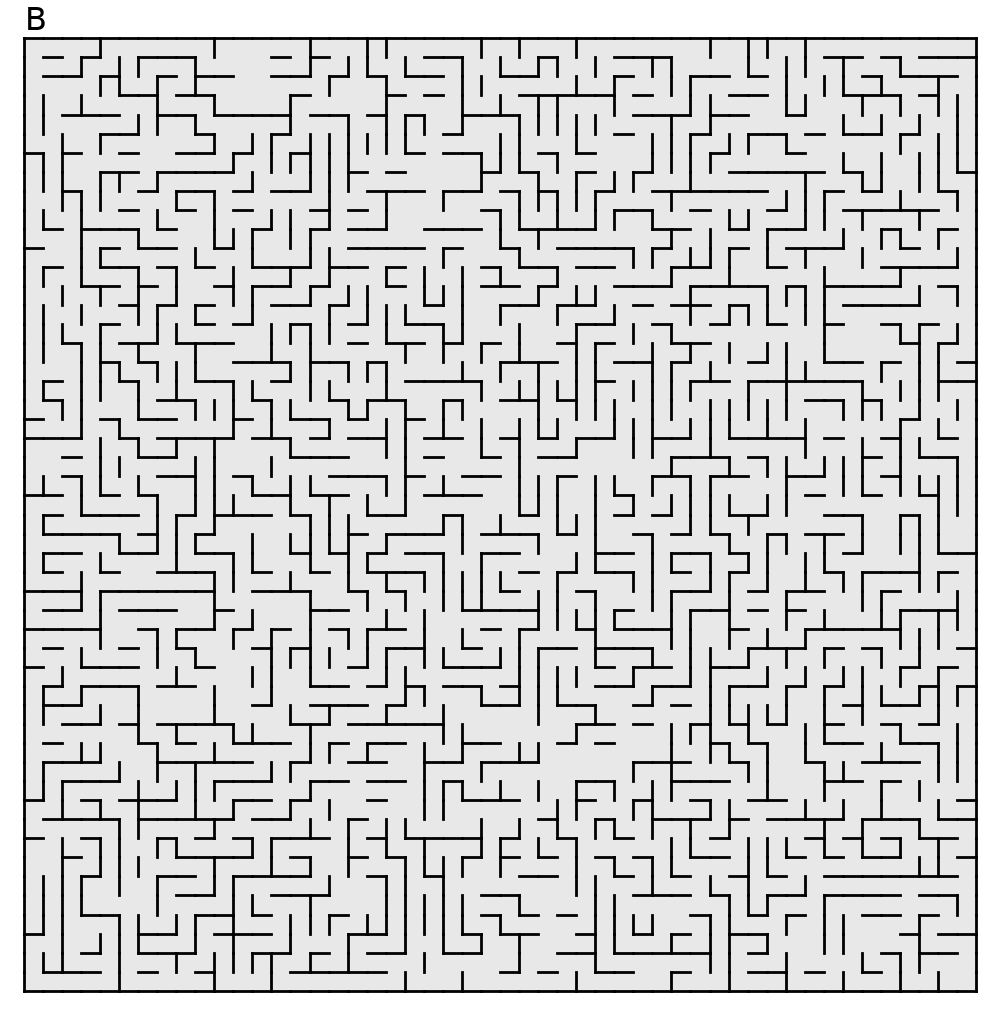}
		\includegraphics[width=0.33\linewidth]{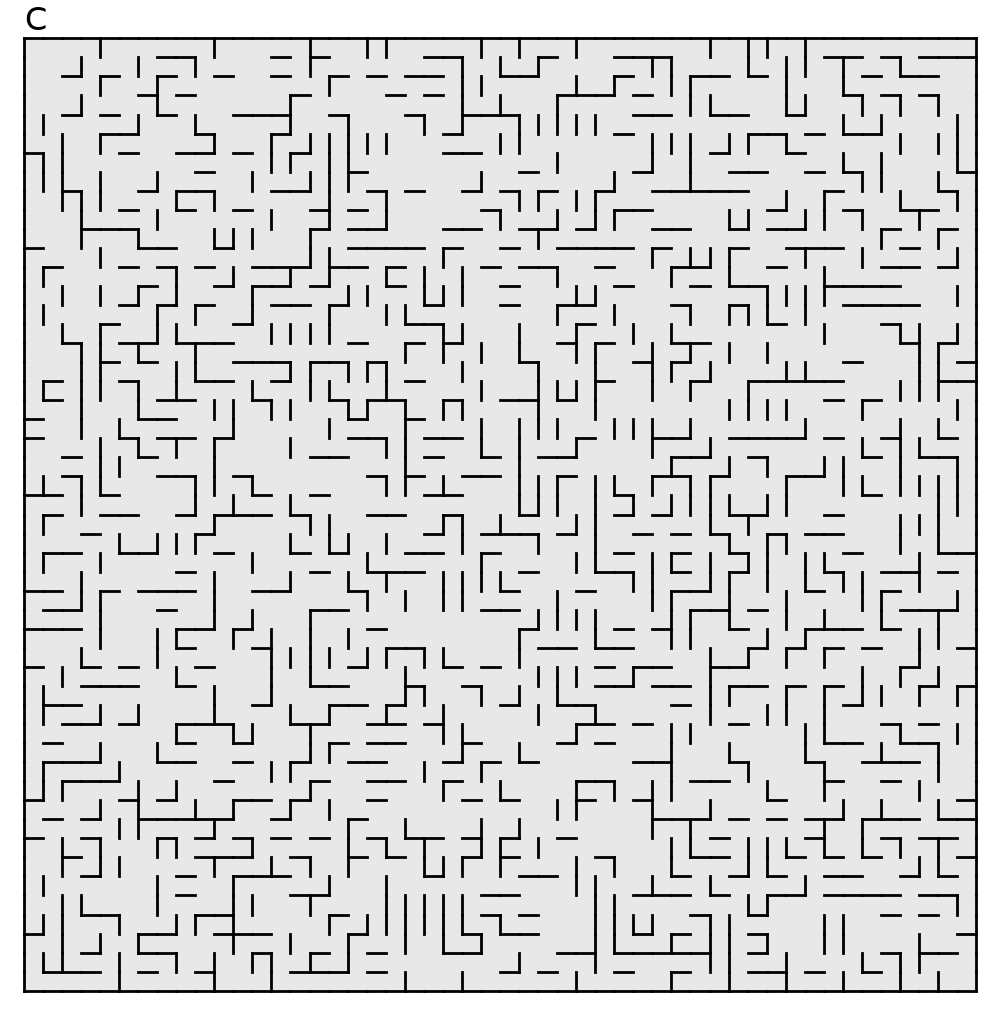}
		\caption{Examples of $50 \times 50$ mazes with different wall densities: a perfect maze  with 49\% (A), 40\% (B) and 30\% (C).}
		\label{fig:maze_50x50_wall_density}
	\end{figure*}
	
	To properly evaluate the HEDAC maze exploration algorithm, we conducted many maze exploration simulations, varying different exploration methods, maze layouts, the number of agents and their initial positions, and algorithm settings such as anti collision. In this section, we will describe detailed information about the experimental environment, the datasets used, and the methods used to interpret the results specifically tailored to each test type.
	
	The number of maze exploration simulations performed for each combination of maze configuration, number of agents and algorithm setup is given in Table~\ref{tab:experimental_environment}.
	
	A total of 6 different shapes of mazes were used in the simulations performed, namely $10 \times 10$, $20 \times 20$, $50 \times 50$, $100 \times 100$, $200 \times 150$ and $400 \times 150$, as shown in Table~\ref{tab:experimental_environment}. Initially, maze layouts are generated using the backtracking algorithm implemented in the Mazelib Python module. Each of the generated layouts represents a perfect maze in which any two maze nodes can be connected by a unique path \citep{bellot2021generate}. Exploring a perfect maze is not suitable for evaluation as it only allows limited movement decisions. For this reason, we consider sparser mazes formed by randomly tearing down walls in a perfect maze. We introduce a wall density that represents the fraction of walls present in the maze and the total number of possible walls (if each maze node is completely enclosed by walls). Examples of $50 \times 50$ mazes with different wall densities are shown in Figure~\ref{fig:maze_50x50_wall_density}.

	For each maze shape, a given number of layouts and initial configurations of the agents are randomly generated. The initial configuration consists of the generated initial positions of the agents, which ensure that there are no overlaps (at most one agent can be located at a node).
	
	Three algorithms for maze exploration are investigated in numerical experiments: Kivelevitch and Cohen, Alian and the proposed HEDAC technique. Moreover, we also investigate different properties of the HEDAC maze algorithm, specifically the motion constraint for collision avoidance and scalability.
	
	For all maze exploration simulations, we use the same settings for all variants of the HEDAC algorithm. The cooling factor is set to $\alpha=0.3$. The iterative linear SOR solver is set to maximum relative error $\varepsilon=10^{-4}$ and $\omega=1.4$. Although we explore all variants to some extent, the SOR solver with collision avoidance mechanism and unknown maze layout is the referent configuration variant of the proposed maze exploration algorithm, that is evaluated in each analysis in the presentation of results.
	
	\subsection{Results and Discussion}
	\label{sec:results_discussion}
	\subsubsection{Maze explorations results}
	\label{sec:expl_results}
	\begin{figure*}[hb!]
		\centering
		\includegraphics[width=0.33\linewidth]{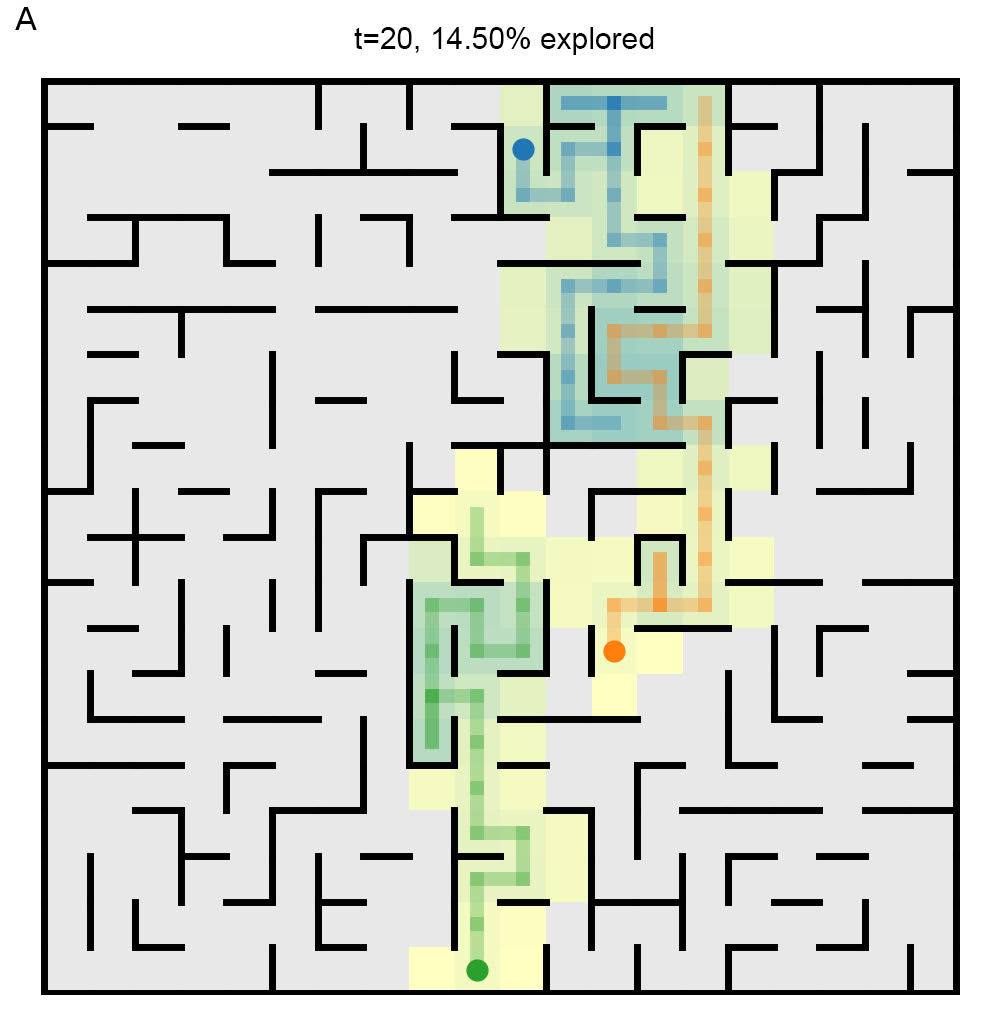}
		\includegraphics[width=0.33\linewidth]{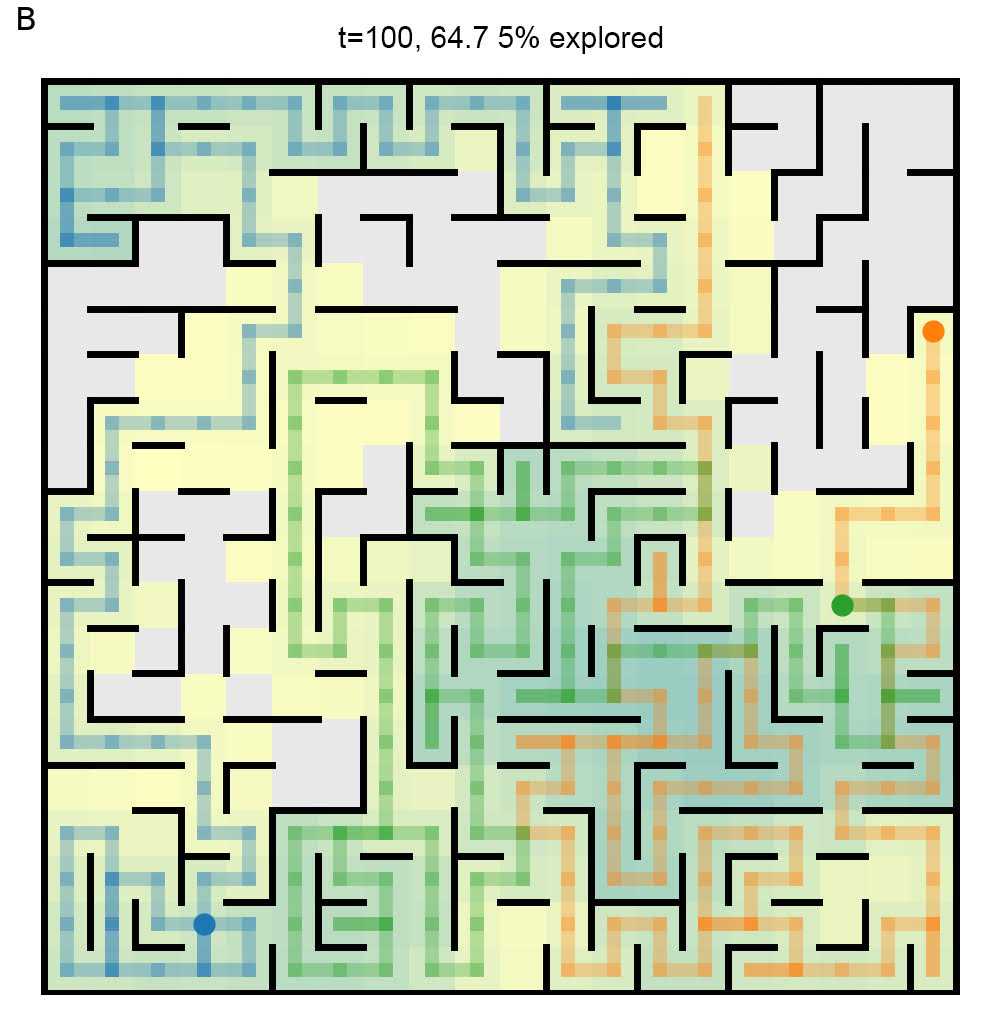}
		\includegraphics[width=0.33\linewidth]{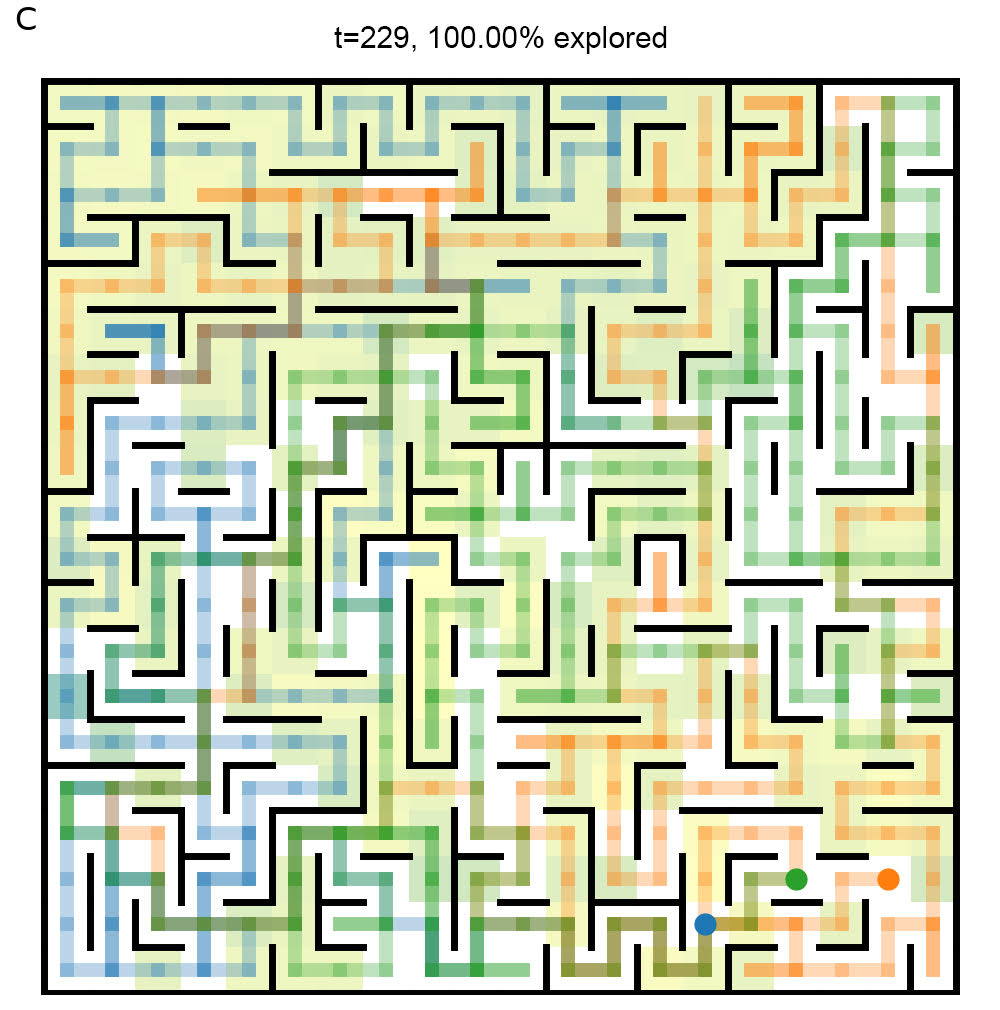}
		\caption{Example of a simulation of $20 \times 20$ maze exploration by three agents. (A), (B) and (C) show agent trajectories and potential in known nodes after 20, 100 and 229 steps, respectively. Grey colored nodes are not known by the algorithm at given step. }
		\label{fig:detailed_20x20}
	\end{figure*}

	We choose a randomly generated layout for a $20\times 20$ maze to demonstrate exploration by the proposed HEDAC algorithm. The maze is explored by three randomly positioned agents, and all maze nodes are visited after 229 steps (Figure~\ref{fig:detailed_20x20}). The progress of the exploration can be observed in the trajectories and visible nodes shown in the subfigures (A-C). As more parts of the maze are explored, the more frequently and inevitably existing paths are traversed. This leads to a stagnation of the coverage, as can be seen in Figure~\ref{fig:detailed_20x20_convergence}.
	
	The computing power is shown in the first diagram in Figure~\ref{fig:detailed_20x20_convergence}, where the calculation of the entire step takes about 1 ms CPU time. Only 0.280 s are required to fully explore the maze.
	The demonstrated computational efficiency is mainly achieved by using an iterative BR SOR linear system solver, which in most cases can provide results within the prescribed accuracy of $\varepsilon=10^{-4}$ in a single iteration. Occasionally, a few more iterations are required if the potential field changes significantly. This happens when the agent explores the last node of a closed maze corridor, resulting in the potential in the entire corridor being close to zero. Although in general SOR methods can converge slowly in this application, the SOR iterations have a good starting point and the relative accuracy of the solution is important only near the agent. 
	
	\begin{figure}[hb!]
		\includegraphics[width=\linewidth]{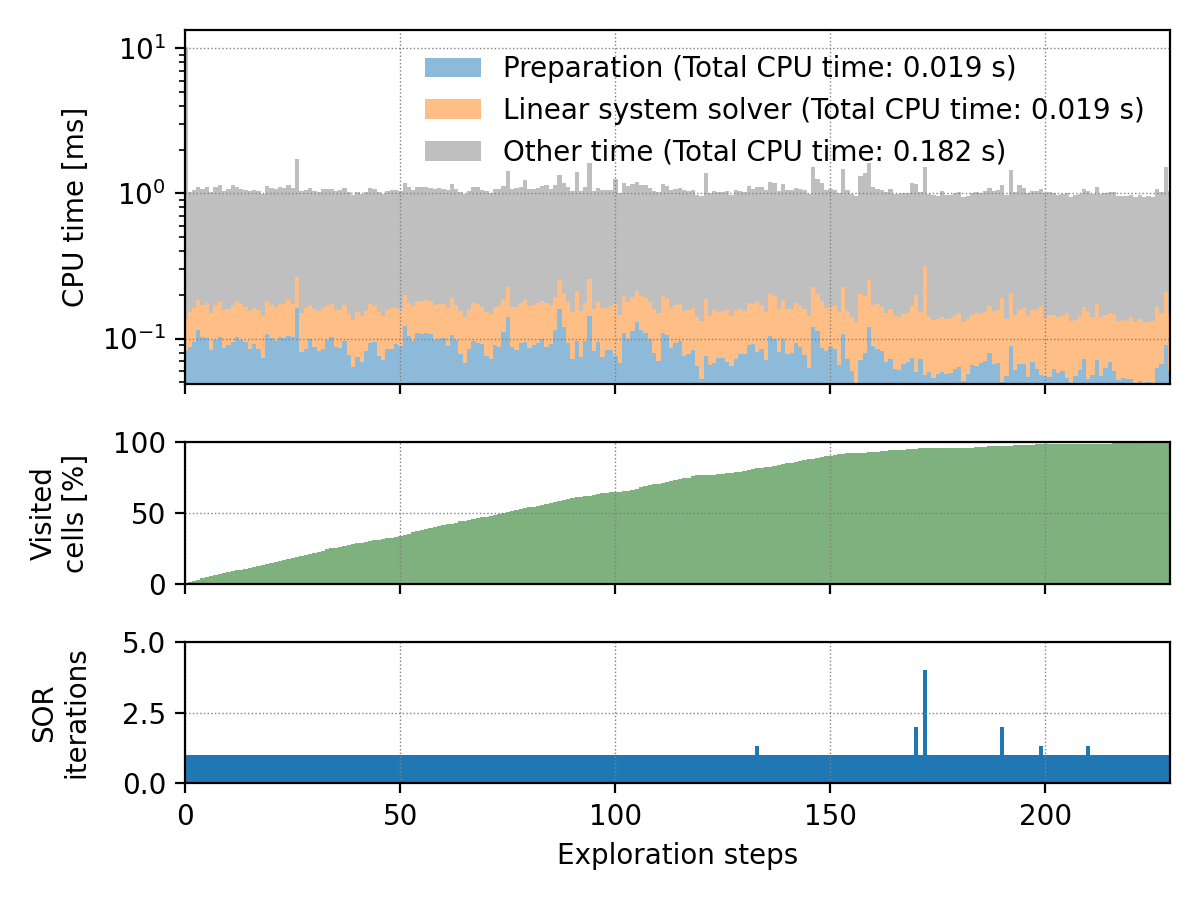}
		\caption{Example of a simulation of $20 \times 20$ maze exploration by three agents: Computational times, coverage, and number of iterations for SOR solver are displayed in for each step of the exploration.}
		\label{fig:detailed_20x20_convergence}
	\end{figure}
	Additionally, we conducted 500 exploration tests on two 
	$10\times10$ mazes: one being a perfect maze and the other a more passable maze with a $30\% $ wall density. We prepared a dataset with 500 random initial positions for five agents and tracked the percentage of visited nodes during exploration.

	\begin{figure}[h!]
		\centering
		\begin{subfigure}[b]{0.5\textwidth}
			\includegraphics[width=0.9\textwidth]{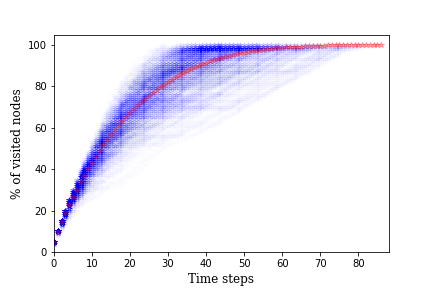}
			\caption{10x10  perfect maze}
			\label{fig:perc45} \end{subfigure}	
		\begin{subfigure}[b]{0.5\textwidth}
			\includegraphics[width=0.9\textwidth]{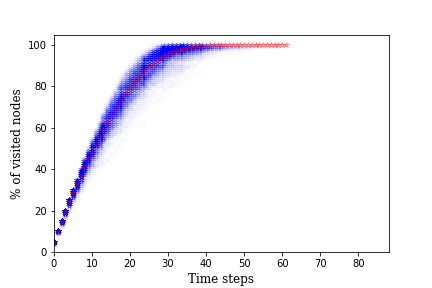}
			\caption{10x10  maze with density 30$\%$ }
			\label{fig:perc30} \end{subfigure}	
		\caption{Percent of visited nodes of a 10x10 mazes with 45$\%$ and 30$\%$ density during the exploration of the maze by 5 agents when  AC is turned off and  $\alpha=0.05$.} 
		\label{fig:perc}
	\end{figure} 
	Figure \ref{fig:perc} presents the results of those 500 conducted tests, illustrating the differences in exploration efficiency between the two maze types.
	It shows that in the first half of the simulation for both mazes the average percent of visited nodes increases almost linearly and the progress slows down in the last half. The results for fewer obstacles are significantly better on average and the variance is obviously lower. This test shows that the exploration of a prefect maze is more sensitive to the initial positions of the agents compared to sparse mazes. This test confirms the intuition that a higher density of obstacles in a maze increases the average number of steps required to explore the maze and increases the sensitivity of the search to the initial positions of the agents.

	\subsubsection{Validation of iterative BR  SOR solver}
	\label{sec:iterative}
	The linear solver is used to compute the potential field before each agent takes a next step, and this is the bottleneck of the presented algorithm. Therefore, the choice of the linear solver is crucial for the overall complexity of the algorithm. We compare a classical direct Gaussian solver calling a LAPACK routine $dgesv$ with the BR SOR method, described by equations \eqref{eq:BackiterationSOR} and \eqref{eq:RediterationSOR}, which is implemented using standard NumPy arrays. Although the BR SOR method is not precompiled, it can exploit the parallelization of  NumPy  array operations. For a fair comparison, we have chosen a relatively large 50x50 maze with 50 agents with fixed initial positions and $\alpha=1.4$. The number of steps with two solvers is not equal, as rounding errors lead to different decisions at maze intersections. Iterative SOR solver produces larger errors because of the high error tolerance but lowering the error tolerance  only slows down the solver but does not improve the search.
	
	The total computational cost of the direct solver is significantly higher than that of the iterative solver. Figure \ref{fig:computationalCost} shows the CPU time at each time step for the maze mapping using different linear solvers. The iterative solver is obviously faster at each time step. Moreover, the total time to prepare and solve the linear system with the direct solver is $541.48 s$, which is much higher than the $1.49 s$ of the iterative BR SOR solver. From an additional comparison (Figure ~\ref{fig:iterative_vs_direct}) of the total computational cost and the average number of time steps obtained for a data set described in Table ~\ref{tab:experimental_environment}, it can be seen that the iterative BR SOR solver has a clear advantage in terms of the required computational cost. We should emphasize that the expected complexity of the direct solver is $\mathcal{O}(n^2)$, where $n$ is the number of visible nodes which means we can expect computational cost to increse as more nodes are discovered. This expectation is confirmed by the slope of the linear regression line in Figure \ref{fig:convergenceD}. The complexity of the BR SOR solver is theoretically more difficult to determine, but should at best be $\mathcal{O}(n*iterations)$. The slope in Figure \ref{fig:convergenceSOR} shows that the complexity of the linear solver increases very slowly because number of iterations converges on average in one iteration step.
	  
	The iterative solver uses only vector operations, exploits parallelization and has a lower memory memory load then direct solver since the linear system is not explicitly constructed. Due to these advantages, only the iterative solver is used in all simulations of maze exploration in this section. It should be noted that, other iterative methods where tested but the implemented BR SOR emerged as a better option, as it avoids overheads associated with rearranging the system when the domain changes. Moreover, for a fixed sized linear system BR SOR outperformed iterative biconjugate gradient method (BICG) and generalized minimal residual method (GMRES), as well as direct DGESV method, but the results are omitted.
	\begin{figure}[h!]
		\centering
		\begin{subfigure}[b]{0.5\textwidth}
			\includegraphics[width=0.9\textwidth]{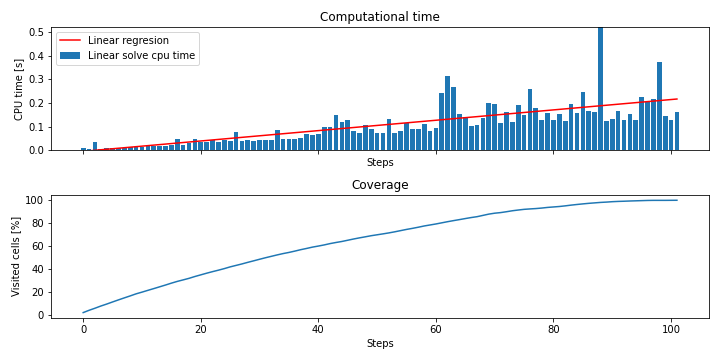}
			\caption{Direct solver computational cost}
			\label{fig:convergenceD} \end{subfigure}	
		\begin{subfigure}[b]{0.5\textwidth}
			\includegraphics[width=0.9\textwidth]{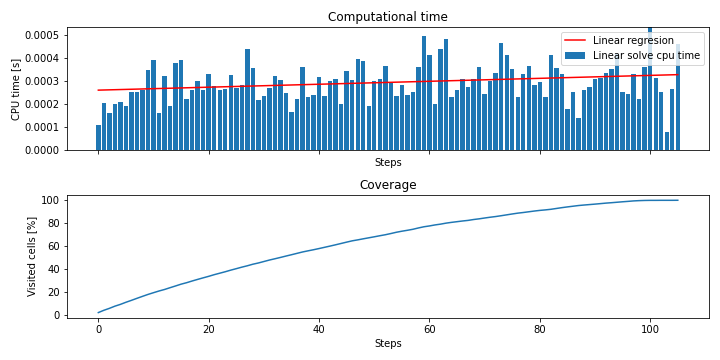}
			\caption{BR SOR iterative solver computational cost}
			\label{fig:convergenceSOR} \end{subfigure}	
		\caption{Direct and iterative solver computational cost for $50x50$ size maze with $30\%$ obstacles and 50 agents.} 
		\label{fig:computationalCost}
	\end{figure} 
	
	\begin{figure*}[h!]
		\centering
		\includegraphics[width=\linewidth]{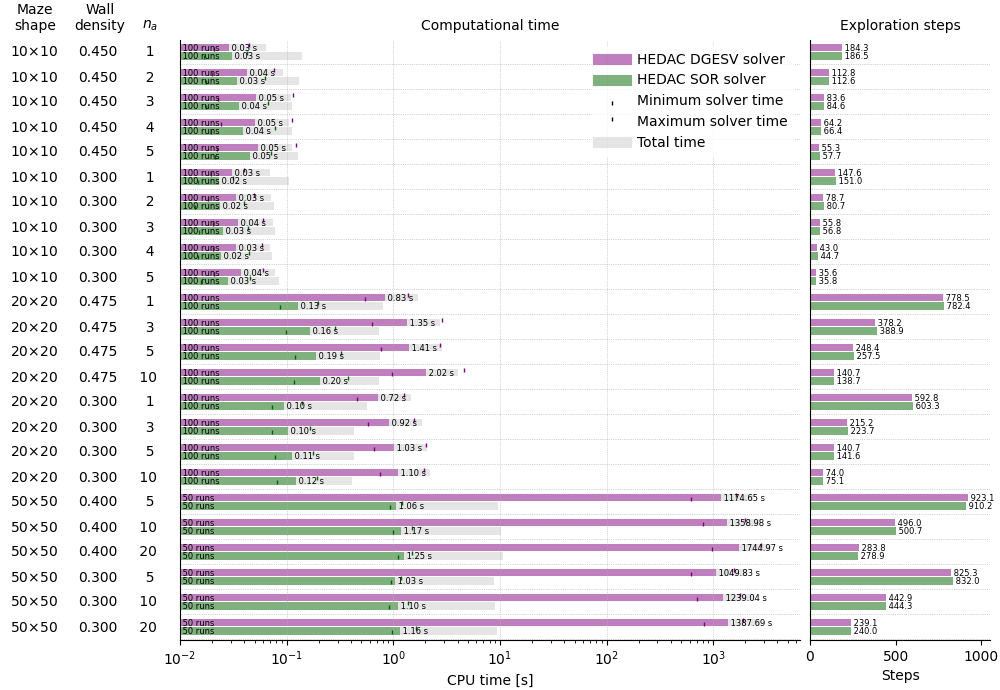}
		\caption{Comparison of iterative BR SOR and direct solver DGESV linear system solvers.}
		\label{fig:iterative_vs_direct}
	\end{figure*}

	\subsubsection{Collision avoidance}
	\label{sec:colision}
	
	\begin{figure}[h!]
		\centering
		\includegraphics[width=\linewidth]{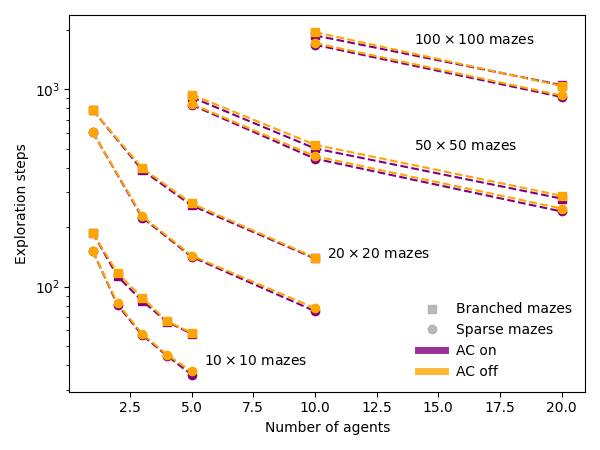}
		\caption{Influence of collision avoidance mechanism the performance of the HEDAC maze exploration algorithm.}
		\label{fig:compare_collision_avoidance}
	\end{figure}
	In the presented algorithm, a linear system is solved and the potential field is updated before each agent chooses the next step. For this reason, it is useful to check the influence of the collision avoidance mechanism.
	We tested the behavior of the algorithm with AC on and off in perfect mazes and general mazes with 30$\%$ wall density. For the tests, we used the dataset described in section \ref{sec:exp_env} and the results of the tests performed can be seen in Figure \ref{fig:comparison_alternative}.
	The Wilcoxon test with a significance level of $p=0.05$ shows that there is no significant difference between the number of steps required to explore the entire maze with AC off and on.
	Another plot of the performance of the HEDAC maze exploration algorithm as a function of the collision avoidance mechanism can be seen in Figure \ref{fig:compare_collision_avoidance}. It is noticeable that the performance of the algorithm is almost the same regardless of whether AC is off or on.
	These test results show that the results of the HEDAC algorithm do not require AC on to ensure an efficient workload distribution. If the potential field was updated only once per time step and the agents chose their next step independently, the AC would be necessary to prevent the agents' paths from merging and traveling together from then on.
	
	\subsubsection{Comparison with alternative algorithms}
	\label{sec:comparison_alternative}
	As can be seen from the Table ~\ref{tab:literature_overview_maze} in Section ~\ref{sec:introduction}, there are not many algorithms that deal with searching in an unknown maze with the goal of finding a target with an unknown location. We have performed two tests in which we compare the HEDAC algorithm with two alternative algorithms presented in \cite{kivelevitch2010multi} and \cite{alian2022multi}. We compare the efficiency of the algorithms in terms of the average time steps required for exploring the entire maze, i.e. mapping the entire maze, and the average number of steps required for the first visit to a given node in each scenario. We chose to use the Wilcoxon test to compare both the average number of steps required to explore the maze and the average number of steps required to first visit a particular node in each scenario. This decision was based on the non-parametric nature of the Wilcoxon test, which does not require the assumption of a normal distribution of the data. Given the potential skewness and presence of outliers in the step count data, as well as our sample size, the Wilcoxon test was better suited to accurately assess the central tendencies of our datasets than other statistical tests that assume normality. 
	
	\begin{figure}[hb!]
		\includegraphics[width=\linewidth]{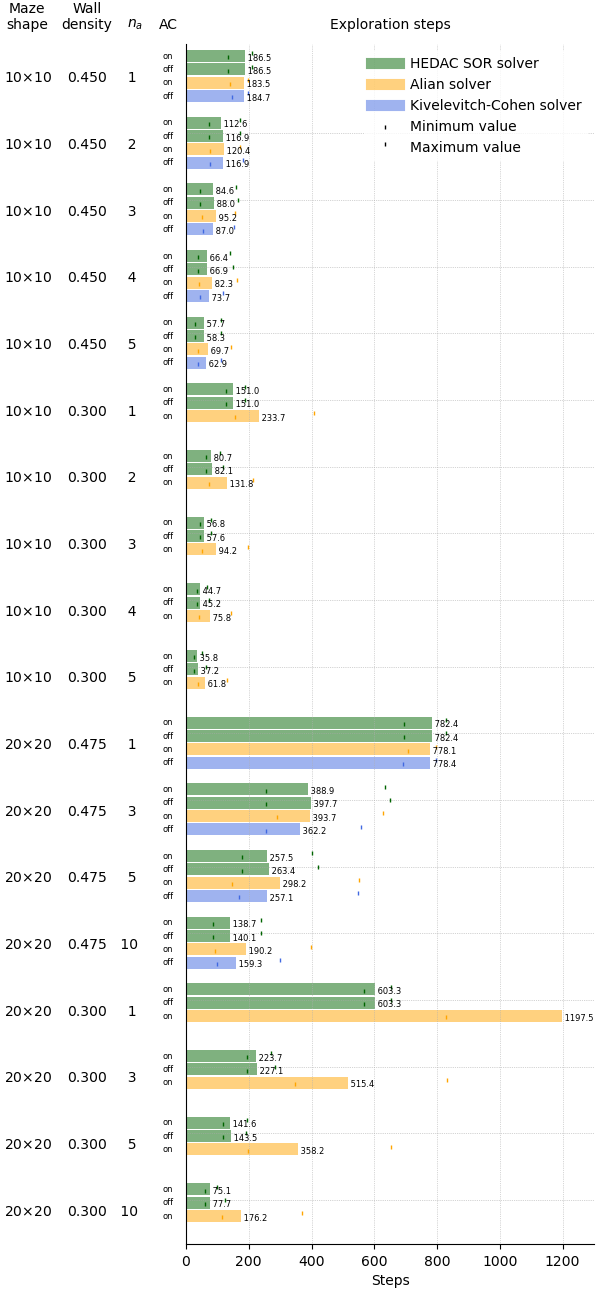}
		\caption{Comparison of the algorithms in terms of the average number of steps required to explore the entire maze.}
		\label{fig:comparison_alternative}
	\end{figure}
	\medskip
	\textit{Comparison with Kivelevitch and Cohen algorithm }\\
	Kivelevitch and Cohen's algorithm is used to solve a problem in which a group of agents searches for an exit of the maze (first phase) and, when it has found it, leads the entire group out of the maze through this exit (second phase). In the first phase, i.e. when searching for the exit, each agent follows a generalization of the Tarry algorithm, except that all agents know which node has already been visited. Each agent follows the following rules:
	\begin{itemize}
		\item The agent should move to nodes that have not yet been traveled to by another agent.
		\item If there are several such nodes, the agent should choose one at random.
		\item If there is no unvisited node (at least one agent has visited all candidate nodes), the agent should prefer a node that has not yet been visited by it.
		\item If the agent has already moved in all possible directions or has reached a dead end, it should backtrack to a node that fulfills one of the previous conditions
		\item All steps should be logged by the agent as it moves.
		\item When retreating, the nodes from which the agent has retreated should be marked as "dead end"
	\end{itemize}

	In order to achieve the completion of the second phase, i.e. to achieve that all agents leave the maze, the Kivelevitch-Cohen algorithm had to ensure that the agents have at least one common point in their paths, which was easiest to achieve by having all agents start from the same position.
	Since our goal is to find the target, we used only the first phase of Kivelevitch and Cohen's algorithm to compare with our results, so that it is possible to place the agents in different starting positions.
	To compare the HEDAC maze exploration algorithm  with the algorithm implemented by Kivelevitch and Cohen, we performed tests with 10x10 and 20x20 perfect mazes as described in Table ~\ref{tab:experimental_environment}. We used the HEDAC algorithm to explore mazes with the AC off, since Kivelevitch and Cohen's algorithm works under the same conditions. We tested Kivelevitch and Cohen's algorithm ( using the same data set). As can be seen from Figure ~\ref{fig:comparison_alternative}, the HEDAC algorithm for exploring the maze has no particular advantage over the Kivelevitch and Cohen's algorithm in terms of the average time steps required to map each node of the maze. According to the Wilcoxon test with a $p$ value of $0.05$, the differences in the average number of steps required by both algorithms to explore the maze were not significant, except in the case of exploring a 10x10 maze with 5 agents and a 20x20 maze with 10 agents, where the average number of steps to explore the maze was lower for the HEDAC maze exploration algorithm. From the results shown in Table ~\ref{tab:cohen}, it can be seen for how many nodes there is a  significant difference in the number of steps required for the first visit, depending on the algorithm used, with a $p$- value of $0.05$ and it is noted how many of these nodes were visited by the HEDAC maze exploration algorithm in fewer steps on average. Based on this data, it can be concluded that HEDAC performs better compared to Kivelevitch and Cohen's algorithm when the number of agents is increased. Better performance of the HEDAC algorithm could possibly be achieved by adjusting the $\alpha$ parameter, but since we used the settings described in Section ~\ref{sec:exp_env} for all tests, we will not analyze this aspect further here. However, considering that Kivelevitch and Cohen's algorithm was developed for searching in perfect mazes and cannot be applied to other types, we believe that the roughly equivalent results and adaptability of the HEDAC algorithm for exploring mazes to different maze types provide a sufficient advantage over this algorithm.
	\begin{table}[h!]
		\centering\footnotesize
		\caption{Comparation of HEDAC and Kivelevitch-Cohen algorithm on perfect mazes.The total number of nodes with a significantly different average number of time steps for their first visit, and the number of nodes that are, on average, visited earlier by the HEDAC algorithm.  $p=0.05$ }
		\label{tab:cohen}
		\begin{tabular}{ccccc}
			\hline
			Maze size&\multicolumn{2}{c}{$10 \times 10$} &\multicolumn{2}{c}{$20 \times 20$} \\
			Wall density	 & \multicolumn{2}{c}{$30\%$}&  \multicolumn{2}{c}{$30\%$}\\
			%\hline
			Number of agents&	Total&HEDAC & Total&HEDAC\\
			\hline 
			1&33&15&63&32 \\%\hline 
			2&28&8&-& -\\%\hline
			3&34&19&61&29 \\%\hline
			4& 36&28 &-&-\\%\hline
			5& 43&33&131&61 \\%\hline
			10&- & -& 117&96\\ \hline
		\end{tabular}
	\end{table}
	\\
	
	\textit{Comparison with Alian's algorithm}\\
	The Alian algorithm can be seen as an improvement of the Kivelevitch and Cohen algorithms. It works for all types of mazes and solves the problem of finding the target in the maze and then getting the entire group out of the maze through this exit. In this algorithm, the nodes of the maze are colored in three colors: white are unvisited maze nodes, gray are visited maze nodes, and black are maze nodes from which there is only one path to the exit.
	When searching for the exit, assuming that an agent detects a neighboring maze node with its sensor, each agent follows the following rules:
	\begin{itemize}
		\item For each currently visited maze node, the occupied flag is set to true so that other agents know that they cannot go there.
		\item If there is a neighboring maze node that has not yet been visited by an agent (white), the agent should go there. If there are several such nodes, the agent should choose one at random.
		\item If there is no node that has not yet been visited by an agent, the agent should prefer a gray node that has not yet been visited by it.
		\item If there are several such nodes, the agent should choose one at random.
		\item If the agent has already traveled in all possible directions, the agent will choose the gray  node it has visited the least time by it.
		\item If there is only one neighboring maze node that can be visited, but it is currently occupied, the agent should wait at its current maze node until it becomes free
		\item If there is only one way to leave the current node and all other directions are obstacles or black nodes, the agent marks the current node as a black node so that other agents cannot move there.
	\end{itemize}
	A detailed algorithm can be found in \cite{alian2022multi}.
	To compare the HEDAC algorithm for maze exploration with the implemention of Alian algorithm, we performed tests with 10x10 and 20x20 perfect and more passable mazes as described in Table ~\ref{tab:experimental_environment}. 
	We use HEDAC with AC on, since the Alian's algorithm does not impose a collision constraint. From the comparison results shown in Figure ~\ref{fig:comparison_alternative}, it is clear that the HEDAC algorithm requires significantly fewer steps to explore sparse mazes with wall densities of $30\%$. When analyzing in more detail the average number of time steps required to explore mazes and performing the Wilcoxon test with a $p$ value of $0.05$, we found that there is no significant difference in results for 10x10 perfect mazes explored by 1, 2, and 3 agents, nor in results for 20x20 perfect mazes when exploration is performed by  5 and 10 agents.  We also  performed the Wilcoxon test with a $p$ value of $0.05$ to compare the algorithms based on the average number of time steps required to visit each maze node for the first time. In mazes with a 30$\%$ density, the total number of nodes where significant differences in results are observed is about half of the total number of nodes, and for these nodes, in terms of the average number of steps required to visit for the first time  that node, the HEDAC algorithm has a significant advantage. In perfect mazes, a smaller proportion of nodes show a significant difference in results.In the case of 10x10 perfect mazes with 2 and 3 agents, and 20x20 perfect mazes with 1  and 2 agents Alian's algorithm performed better. What is evident is that we have a significant advantage in exploring sparser mazes compared to  implementation of the Alian's algorithm. Additionally, it has been observed that for the Alian algorithm, random bias can significantly influence the total number of steps required to complete the same task.
	
	\begin{table}[h!]
		\centering\footnotesize
		\caption{Comparation of HEDAC and Alian's algorithm on sparse mazes. The total number of nodes with a significantly different average number of time steps for their first visit, and the number of nodes that are, on average, visited earlier by the HEDAC algorithm.  $p=0.05$ }
		\label{tab:alian1}
		\begin{tabular}{ccccc}
			\hline
			Maze size&\multicolumn{2}{c}{$10 \times 10$} &\multicolumn{2}{c}{$20 \times 20$} \\
			Wall density	 & \multicolumn{2}{c}{$45\%$}&  \multicolumn{2}{c}{$47.5\%$}\\
			%\hline
			Number of agents&	Total&HEDAC & Total&HEDAC\\
			\hline 
			1&29&21&117&52 \\%\hline 
			2&27&12&-& -\\%\hline
			3&6&6&256&54 \\%\hline
			4& 36&33 &-&-\\%\hline
			5& 25&20&48&46 \\%\hline
			10& -&- & 79&78\\ \hline
		\end{tabular}
	\end{table}
	
	\begin{table}[h!]
		\centering\footnotesize
		\caption{Comparation of HEDAC and Alian's algorithm on perfect mazes. The total number of nodes with a significantly different average number of time steps for their first visit, and the number of nodes that are, on average, visited earlier by the HEDAC algorithm.  $p=0.05$ }
		\label{tab:alian2}
		\begin{tabular}{ccccc}
			\hline
			Maze size&\multicolumn{2}{c}{$10 \times 10$} &\multicolumn{2}{c}{$20 \times 20$} \\
			Wall density	 & \multicolumn{2}{c}{$30\%$}&  \multicolumn{2}{c}{$30\%$}\\
			%\hline
			Number of agents&	Total&HEDAC & Total&HEDAC\\
			\hline 
			1&41&41&239&239 \\%\hline 
			2&40&39&-& -\\%\hline
			3&56&54&239&239 \\%\hline
			4& 48&46 &-&-\\%\hline
			5& 79&78&268&268 \\%\hline
			10& -& -& 215&214\\ \hline
		\end{tabular}
	\end{table}

	\subsubsection{Scalability}
	\label{sec:scalability}
	\begin{figure}[h!]
		\centering
		\includegraphics[width=\linewidth]{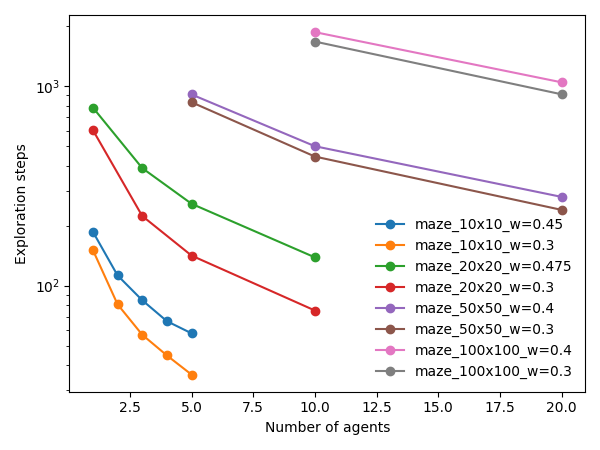}
		\caption{Analysis of the performance of the HEDAC maze exploration algorithm with various number of agents.}
		\label{fig:compare_number_of_agents}
	\end{figure}
	
	We will discuss scalability in terms of dividing the work between multiple agents, more specifically we will discuss scalability in terms of the percentage of average steps required to explore the maze with a multi-agent system compared to the average steps required to explore the maze with a single agent. In Section ~\ref{sec:colision} it is shown that there is no significant difference between the performance of the HEDAC algorithm depending on the collision avoidance mechanism, so it is sufficient to analyze the scalability of the HEDAC algorithm with only one of these two options, i.e. AC on or AC off.
	Table ~\ref{tab:scale30unkoffSE} shows the scalability analysis  of the HEDAC maze exploration algorithm in the sparser mazes in terms of speedup coefficient $S$ and efficiency $E$. The speedup coefficient $S$ stands for the ratio of the average number of time steps an agent needs to explore the entire maze to the average number of time steps $n$ agents need to complete the same task, and the efficiency coefficient $E$ stands for the ratio of speedup to the number of agents. The speedup coefficient $E$ indicates that although performance increases with additional agents, the improvement is not linear to the number of agents. Consequently, each additional agent reduces efficiency.
	The efficiency coefficient $E$, which is relatively high and stable for the results shown in the Table ~\ref{tab:scale30unkoffSE} for up to 5 agents, indicates that the HEDAC algorithm distributes the tasks quite well among small number of agents. However, the decrease in efficiency when exploring the 20x20 maze with 10 agents indicates that the optimal number of agents for 20x20 mazes is up to 10.
	%\begin{table}[h!]	
	%	\centering\footnotesize
	%	\caption{Scalability of HEDAC algorithm  with AC switched off in the unknown maze with 30$\%$ obstacles share.  Average total number of time steps and average total number of agent steps for finding an target with a multi-agent system. }
	%	\label{tab:scale30unkoff}
	%	\begin{tabular}{ccccc}
		%		\hline
		%		Maze size & \multicolumn{2}{c}{$10 \times 10$} &\multicolumn{2}{c}{$20 \times 20$} \\
		%		\hline
		%		Number of & Total time & Total agent & Total time & Total agent \\
		%		agents&steps&steps&steps&steps\\\hline
		%		1& 151.0 & 151.0 & 600.3  &  600.3  \\%\h
		% 	2&80.7 & 161.4 &- &- \\%\hline
		% 	3&56.8 &170.4 &223.7 & 671\\%\hline
		%    4&44.7 & 178.8 &- &  -\\%\hline
		%  5& 35.8& 179 & 141.6& 708 \\%\hline
		%  10& -& - &75.1 & 751\\ \hline
		%\end{tabular}
		%\end{table}
		\begin{table}[h!]	
			\centering\footnotesize
			\caption{Scalability of HEDAC algorithm  with AC switched off in the unknown maze with 30$\%$ obstacles share.  Average total number of time steps, speedup coefficient $S$ , and efficiency coefficient $E$ for different numbers of agents used to explore a maze. }
			\label{tab:scale30unkoffSE}
			\begin{tabular}{ccccccc}
				\hline
				Maze size & \multicolumn{3}{c}{$10 \times 10$} &\multicolumn{3}{c}{$20 \times 20$} \\
				\hline
				Number of &Total& &&Total&&\\
				agents & steps&$S$& $E$ & steps&$S$ & $E$ \\
				\hline
				1& 151.0 & 1 &1& 600.3  &  1 &1 \\%\h
				2&80.7 & 1.87 &0.94&- &-&- \\%\hline
				3&56.8 &2.65&0.88&223.7 & 2.68&0.89\\%\hline
				4&44.7 & 3.37&0.84 &- &  -&-\\%\hline
				5& 35.8& 4.21 & 0.84&141.6& 4.23& 0.85 \\%\hline
				10& -& - &- &75.1& 7.5& 0.75 \\ \hline
			\end{tabular}
		\end{table}
		Figure ~\ref{fig:compare_number_of_agents} shows an analysis of the average number of steps required to explore different types of mazes with different numbers of agents. This analysis shows that as the number of agents increases, the average number of steps required to explore the maze decreases. However, this reduction deviates further from linearity as the maze structure becomes larger and more complex.
		From Figures ~\ref{fig:comparison_alternative} and ~\ref{fig:compare_number_of_agents}, it can be seen that the HEDAC algorithm also works well for perfect mazes and that the division of labour between agents is good. However, the perfect mazes show a distorted picture of the scalability analysis, which is to be expected since the agents in perfect mazes do not have as many opportunities to choose better paths and overtake each other.

		\section{Conclusion}
		\label{sec_conclusion}

		The HEDAC maze exploration algorithm proved to be robust, adaptable, scalable and computationally inexpensive. The algorithm does not need to be adapted for each test, except for the parameter $\alpha$, which emphasizes the global or local search of the multi-agent system. \\ The HEDAC algorithm with the iterative BR SOR linear solver provides a simple and cost-effective solution to the  maze exploration problem that can be used in a dynamically changing multi-agent environment and is capable of adapting to various constraints and applications.    Moreover, the algorithm does not cause deadlocks and guarantees the exploration of the entire maze if all nodes can be reached.
	The anti-collision procedure can still be improved, but we have shown that this will certainly not have a major impact on the overall average result. In future work, we will generalize the algorithm to graph search and monitoring, which is an obvious generalization that allows for many different applications. An additional step to improve this approach could be the theoretical calculation of the optimal relaxation parameter for the SOR method applied to this particular application. However, although the computational load of the BR SOR method is very low, an alternative iterative method should be considered for larger mazes, as the convergence of the BR SOR method slows down with the increasing size of the linear system. The approach taken in this paper can be used to improve similar previous HEDAC implementations in continuous domains as well as  on graphs, and also as motivation for exploring new ways to accelerate the HEDAC method.

		\section*{CRediT authorship contribution statement}
		\textbf{Bojan Crnković}:  Methodology, Conceptualization, Supervision, Writing – original draft, Writing – review
		\& editing, Validation, Software, Formal analysis, Investigation.
		\textbf{Stefan Ivić}: Methodology, Software, Writing – review
		\& editing, Investigation, Supervision, Conceptualization, Visualization.
		\textbf{Mila Zovko}: Writing – original draft, Writing – review
		\& editing, Software, Validation, Methodology, Formal analysis, Investigation.
				
		\section*{Acknowledgements}		
		This publication is supported by the Croatian Science Foundation under the projects UIP-2020-02-5090 (for S.I.) and IP-2019-04-1239 (for B.C. and M.Z.).

		\section*{Data and code availability}
		
		The entire process of mapping the 40x40 maze by 15 agents and the increase in the value of the cumulative function can be viewed on \href{https://youtu.be/SWI9Q_tmPC4?si=5ZXN-jvbB8vVzM-w}{mapping video} and \href{https://youtu.be/9YSC37aZwJs?si=dwYEzYWy9kZcfFNG}{cumulative function video}.
		Process of mapping 20x20 maze by 3 agents where one can track agents trajectories and potential can be viewed on \href{https://youtu.be/jGd8YDygh7M}{potential tracking video}.
		
		Python code along with maze layouts and agent configurations needed to reproduce this research is publicly available at \href{https://gitlab.com/sivic/hedac-maze-exploration}{https://gitlab.com/sivic/hedac-maze-exploration}.

		%% The Appendices part is started with the command \appendix;
		%% appendix sections are then done as normal sections
		%% \appendix
		
		%% \section{}
		%% \label{}
		
		%% If you have bibdatabase file and want bibtex to generate the
		%% bibitems, please use
		%%
		\bibliographystyle{elsarticle-harv} 
		\bibliography{references.bib}
		
		%% else use the following coding to input the bibitems directly in the
		%% TeX file.

	\end{document}